\documentclass[11pt]{article}
\usepackage[margin=1.4in]{geometry}
\usepackage[utf8]{inputenc}

\usepackage{amsmath,amsthm}
\usepackage{amsxtra}
\usepackage{amsfonts,amssymb,bm}
\usepackage{stmaryrd}
\usepackage[normalem]{ulem}
\usepackage{subfigure}
\usepackage{graphicx}
 \graphicspath{./} 
\usepackage{epstopdf}
\usepackage{url}
\usepackage{listings}
\usepackage{xcolor}
\usepackage{hyperref}
\usepackage[capitalize,noabbrev]{cleveref}

\theoremstyle{plain}
\newtheorem{theorem}{Theorem}
\newtheorem{proposition}[theorem]{Proposition}
\newtheorem{lemma}[theorem]{Lemma}
\newtheorem{corollary}[theorem]{Corollary}
\newtheorem{definition}[theorem]{Definition}

\theoremstyle{remark}
\newtheorem{remark}[theorem]{Remark}

\usepackage{algorithm}
\usepackage{algorithmic}

\usepackage{appendix}
\usepackage{braket}
\usepackage{multirow}

\usepackage{comment}

\usepackage{tikz}
\usetikzlibrary{shapes.multipart,calc,fit}
\usetikzlibrary{matrix}

\usepackage{lineno}

\usepackage{braket}
\usepackage{multirow}

\DeclareMathOperator*{\argmin}{arg\,min}

\DeclareMathOperator*{\trace}{tr}
\DeclareMathOperator{\rank}{rank}

\newcommand{\reals}{\mathbb{R}}

\newcommand{\fro}[1]{\|#1\|_{\mathrm{F}}}

\newcommand{\eg}{e.g.}
\newcommand{\ie}{i.e.}

\newcommand{\dop}{\mathrm{D}}
\newcommand{\ddt}{\frac{\mathrm{d}}{\mathrm{d}t}}
\newcommand{\po}[1][\Omega]{\mathcal{P}_{#1}}

\newcommand{\trs}[1][x]{{#1}^{\mathrm{T}}}

\newcommand{\grp}{\mathcal{G}}
\newcommand{\gre}{\mathcal{E}} 
\newcommand{\grv}{\mathcal{V}} 

\newcommand{\rkval}{r}
\DeclareMathOperator{\ad}{ad}

\newcommand{\abs}{\sigma_{\mathrm{abs}}}

\renewcommand{\po}[1][\Omega]{\mathcal{P}_{#1}}
 
\newcommand{\ssgn}[1][]{\mathrm{sign}_{#1}} 
\newcommand{\dom}[1][d]{\reals^{#1\times #1}} 
\newcommand{\tdom}[1][d]{\reals^{#1\times #1}} 
\newcommand{\prodsp}[1][\rkval]{\reals^{d\times #1}\times\reals^{d\times #1}} 
\newcommand{\tol}{\mathrm{tol}}  
\newcommand{\smat}{\mathcal{S}}
\newcommand{\sabs}{\mathcal{S}_{\mathrm{abs}}}
\newcommand{\sodot}{\mathcal{S}_2}

\DeclareMathOperator*{\suppo}{supp}
\newcommand{\exptr}[1][\cdot]{\trace(\exp(#1))}

\newcommand{\splr}{A_{\Omega}}
\newcommand{\setdag}{\mathcal{D}_{d\times d}}
\newcommand{\infn}[1][\cdot]{\|#1\|_{\max}}

\newcommand{\epsh}{\delta_{\epsilon}}
\newcommand{\ca}{C_{0}}
\newcommand{\dparam}{\epsilon^\star_{r,\Omega}}

\newcommand{\michele}[1]{}
 
\def\ourmo{{\sc LoRAM}}      
\def\ouralg{\ourmo-AGD} 
\def\notears{{\sc NoTears}}  

\title{From graphs to DAGs: a low-complexity model and a scalable algorithm}

\author{
Shuyu Dong\footnote{TAU, LISN, INRIA, Universit\'e Paris-Saclay, 91190 Gif-sur-Yvette, France. (\url{shuyu.dong@inria.fr})}
\and 
Mich{\`e}le Sebag\footnote{TAU, LISN, CNRS, INRIA, Universit\'e Paris-Saclay, 91190 Gif-sur-Yvette, France. (\url{michele.sebag@lri.fr})}
}
\date{April 10, 2022}

\begin{document}
\maketitle





\begin{abstract}
Learning directed acyclic graphs (DAGs) is long known a critical challenge at the core of probabilistic and causal modeling. 
The \notears\ approach of Zheng et al.~\cite{NEURIPS2018_e347c514}, 
through a differentiable function involving the matrix exponential trace
$\trace(\exp(\cdot))$, opens up a way to learning DAGs via continuous
optimization, though with an $O(d^3)$ complexity in the number $d$ of nodes. 
This paper presents a low-complexity model, called \ourmo\ for Low-Rank Additive
Model, which combines low-rank matrix factorization with a sparsification
mechanism for the continuous optimization of DAGs.
The main contribution of the approach lies in an efficient gradient
approximation method leveraging the low-rank property of the model, and its
straightforward application to the computation of projections from graph
matrices onto the DAG matrix space.
The proposed method achieves a reduction from a cubic complexity to quadratic
complexity while handling the same DAG characteristic function as \notears\,
and scales easily up to thousands of nodes for the projection problem. 
The experiments show that the \ourmo\ achieves efficiency gains of orders of
magnitude compared to the state-of-the-art at the expense of a very moderate
accuracy loss in the considered range of sparse matrices, and with a low
sensitivity to the rank choice of the model's low-rank component. 

\end{abstract}

%
%

\section{Introduction} 
\label{sec:intro}

The learning of directed acyclic graphs (DAGs) is an important problem for
probabilistic and causal inference~\cite{pearl2009causality,peters2017elements} with important applications 
in social sciences~\cite{morgan2015counterfactuals}, genome research~\cite{stephens2009bayesian} and machine learning itself~\cite{peters2016causal,arjovsky2019invariant,Sauer2021ICLR}. 
Through the development of probabilistic graphical models~\cite{pearl2009causality,buhlmann2014cam}, DAGs are a most natural
mathematical object to describe the causal relations among a number of variables. 
In today's many application domains, the estimation of DAGs faces intractability issues as an ever growing number $d$ of variables is considered, due to the fact that estimating DAGs is NP-hard~\cite{chickering1996learning}. %
The difficulty lies in how to enforce the acyclicity of graphs.
Shimizu et al.~\cite{shimizu2006linear} combined independent component analysis
with the combinatorial linear assignment method to optimize a linear causal model (LiNGAM) and later proposed a direct and sequential algorithm~\cite{shimizu2011directlingam} guaranteeing global optimum of the LiNGAM, for $O(d^4)$ complexities.

Recently, Zheng et al.~\cite{NEURIPS2018_e347c514} proposed an optimization approach to learning DAGs. %
The breakthrough in this work, called \notears, comes with the characterization of the DAG matrices by the zero set of a real-valued differentiable function on
$\reals^{d\times d}$, which shows that an $d\times d$ matrix $A$ is the adjacency matrix of a DAG if and only if the {\it exponential trace} satisfies 
\michele{Je pensais que c'était exp(A). Et que le $A \odot A$ avait été rajouté pour simplifier l'optim.}
\begin{align}\label{eq:zheng18-thm1}
h(A):=\trace(\exp(A\odot A))=d,
\end{align}
and thus the learning of DAG matrices can be cast as a continuous optimization problem subject to the constraint
$h(A) = d$. %
The \notears\ approach broadens the way of learning complex causal relations and provides promising perspectives 
to tackling large-scale inference problems~\cite{kalainathan2018structural,yu2019dag,zheng2020learning,ng2020role}. %
However, \notears\ is still not suitable for large-scale applications as the complexity of computing the exponential trace and its gradient is $O(d^3)$. %
More recently, Fang et al.~\cite{fang2020low} proposed to represent DAGs by low-rank matrices with both theoretical and empirical validation of the low-rank assumption for a range of graph models. However, the adaptation of the \notears\ framework~\cite{NEURIPS2018_e347c514} to low-rank model still yields a complexity of $O(d^3)$ due to the DAG characteristic function in~\eqref{eq:zheng18-thm1}.

The contribution of the paper is to propose a new computational framework to tackle the scalability issues faced by the low-rank modeling of DAGs. %
We notice that the Hadamard product $\odot$ in characteristic functions as in~\eqref{eq:zheng18-thm1} poses real obstacles to scaling up the optimization of \notears~\cite{NEURIPS2018_e347c514} and \notears-low-rank~\cite{fang2020low}. %
To address these difficulties, we present a low-complexity model, named {\em Low-Rank Additive Model} (\ourmo), which is a composition of low-rank matrix factorization with sparsification, and then propose a novel approximation method compatible with \ourmo\ to compute the gradients of the exponential trace in~\eqref{eq:zheng18-thm1}. 
Formally, the gradient approximation---consisting of matrix computation of the form $(A,C,B)\to (\exp(A)\odot C)B$, where $A,C\in\dom$ and $B$ is a thin low-rank matrix---is inspired from the numerical analysis of~\cite{al2011computing} for the matrix action of $\exp(A)$. 
We apply the new method to the computation of projections from graphs to DAGs through optimization with the differentiable DAG constraint. 

Empirical evidence is presented to identify the cost and the benefits of the approximation method combined with Nesterov's accelerated gradient descent~\cite{nesterov1983}, depending on the considered range of problem parameters (number of nodes, rank approximation, sparsity of the target graph). 

The main contributions are summarized as follows: 
\begin{itemize}
    \item The \ourmo\ model, combining a low-rank structure with a flexible sparsification mechanism, is proposed to represent DAG matrices, together with a DAG characteristic function generalizing the exponential trace of \notears~\cite{NEURIPS2018_e347c514}. 

    \item An efficient gradient approximation method, exploiting the low-rank and sparse nature of the \ourmo\ model, is proposed. Under the low-rank assumption ($r\leq
C\ll d$), the complexity of the proposed method is quadratic ($O(d^2)$) instead
of $O(d^3)$ as shown in \Cref{tab:compl}. Large efficiency gains, with insignificant loss of accuracy in some cases, are demonstrated experimentally in the considered range of application. 

\begin{table}[htbp]
\small
\centering 
\caption{Computational properties of \ourmo\ and algorithms in related work. } 
\label{tab:compl}
\begin{tabular}{l|ccc}
\hline\hline
~                                             & Search space$\quad$  & Memory req.$\quad$ & Cost for $\nabla h$\\ \hline
\notears~\cite{NEURIPS2018_e347c514}          & $\dom$   & $O(d^2)$    & $O(d^3)$                \\ 
\notears-low-rank~\cite{fang2020low} &$\prodsp$ & $O(d\rkval)$ & $O(d^3)$                \\ 
\ourmo\ (ours)     &$\prodsp$ & $O(d\rkval)$ &  $O(d^2 r)$\\ 
\hline\hline
\end{tabular}
\end{table}

\end{itemize}


\section{Notation and formal background} 
\label{sec:prelims}

A graph on $d$ nodes is defined and denoted as a pair $\grp=(\grv,\gre)$, where
$|\grv| = d$ and $\gre\subset \grv\times \grv$. By default, a directed graph is simply referred to as a graph. The adjacency matrix of a graph $\grp$, denoted as $\mathbb{A}(\grp)$, is defined as
the matrix such that $[\mathbb{A}(\grp)]_{ij} =1 $ if
$(i,j)\in\gre$ and $0$ otherwise. 
Let $A\in\dom$ be any {\it weighted} adjacency matrix of a graph $\grp$ on $d$
nodes, then by definition, the adjacency matrix $\mathbb{A}(\grp)$ indicates the
nonzeros of $A$; the adjacency matrix $\mathbb{A}(\grp)$ is also the
{\it support} of $A$, denoted as $\suppo(A)$, \ie, $[\suppo(A)]_{ij} = 1$ if $A_{ij} \neq 0$ and $0$ otherwise. The number of nonzeros of $A$ is denoted as $\|A\|_0$. %
The matrix $A$ is called a {\it DAG matrix} if $\suppo(A)$ is the adjacency
matrix of a directed acyclic graph (DAG). 
We define, by convention, the set of DAG matrices as follows: 
$\setdag = \{A\in\dom: \suppo(A) \text{~defines a DAG} \}$. 

We recall the following theorem that characterizes acyclic graphs using the
matrix {\it exponential trace}---$\trace(\exp(\cdot))$---where $\exp(\cdot)$
denotes the matrix exponential function. The matrix exponential will be
denoted as $e^{\cdot}$ and $\exp(\cdot)$ indifferently. The 
operator $\odot$ denotes the matrix Hadamard product that acts on two matrices of the same size by elementwise multiplications. %

\begin{theorem}[{\cite{NEURIPS2018_e347c514}}]\label{thm:zheng18-thm1}
    A matrix $A\in\dom$ is a DAG matrix if and only if 
    \[\trace(\exp(A\odot A)) =d.\] 
\end{theorem}
The following corollary is a straightforward extension of the theorem above: 
\begin{corollary}\label{thm:dag-hfunc}
Let $\sigma:\dom\mapsto\dom$ be an operator such that: (i) $\sigma(A)\geq 0$ and
(ii) $\suppo(A) = \suppo(\sigma(A))$, for any $A\in\dom$. Then, $A\in\dom$ is a DAG matrix if and only if $\trace(\exp(\sigma(A))) =d$. 
\end{corollary}
In view of the property above, we will refer to the composition of $\exptr$ and
the operator $\sigma$ as a {\it DAG characteristic function}. %
Next, we show some more properties (proof in \Cref{sec-app:proofs}) of the exponential trace. 
\begin{proposition}\label{prop:trexpm-divers}
    The exponential trace $\tilde{h}:\reals^{d\times d} \mapsto \reals: A\to\exptr[A]$ satisfies: 
    (i) For all $\bar{A}\in\reals^{d\times d}_+$, $\exptr[\bar{A}] \geq d$ and
    $\exptr[\bar{A}]=d$ if and only if $\bar{A}$ is a DAG matrix. 
    (ii) $\tilde{h}$ is nonconvex on $\reals^{d\times d}$. 
    (iii) The Fr\'echet
derivative of $\tilde{h}$ at $A\in\dom$ along any direction
$\xi\in\reals^{d\times d}$ is \[\dop \tilde{h}(A)[\xi] = \trace(\exp(A)\xi),\] and the
gradient of $\tilde{h}$ at $A$ %
is $\nabla \tilde{h}(A) = \trs[(\exp(A))]$. 
\end{proposition}

\section{LoRAM: a low-complexity model}
\label{ssec:splr-notears}

In this section, we describe a low-complexity matrix representation of the
adjacency matrices of directed graphs, and then a generalized DAG characteristic
function for the new matrix model. %

In the spirit of searching for best low-rank singular value decompositions and taking inspiration from~\cite{fang2020low}, the search of a full $d \times d$ (DAG) matrix $A$ is replaced by the search of a pair of thin factor matrices $(X,Y)$, in $\reals^{d\times \rkval}\times \reals^{d\times \rkval}$ for $1\leq\rkval < d$, and the $d\times d$ candidate graph matrix is represented by the product $X\trs[Y]$. This matrix product has a rank bounded by $\rkval$, with number of parameters $2dr\leq d^2$. 
However, the low-rank representation $(X,Y)\to X\trs[Y]$ generally gives a
dense $d\times d$ matrix. %
Since in many scenarios the sought graph (or Bayesian network) is usually sparse, we apply a {\it sparsification} operator on $X\trs[Y]$ in order to trim abundant entries in $X\trs[Y]$. Accordingly, we combine the two operations and introduce the following model. 

\begin{definition}[LoRAM]\label{def:splr-rep}
Let $\Omega\subset [d]\times [d]$ be a given index set. 
The {\em low-rank additive model} (LoRAM), noted $\splr$, is defined from 
the matrix product of $(X,Y)\in\prodsp$ sparsified according to $\Omega$: 
\begin{equation}
    \label{eq:def-splr-rep}
    \splr(X,Y) = \po(X\trs[Y]), 
\end{equation}
where $\po:\dom\mapsto\dom$ is a mask operator such that 
$[\po(A)]_{ij} = A_{ij}$ if $(i,j)\in\Omega$ and $0$ otherwise. The set $\Omega$ is referred to as the {\em candidate set} of \ourmo.
\end{definition}

The candidate set $\Omega$ is to be fixed according to the specific problem. In the case of projection from a given graph to the set of DAGs, $\Omega$ can be fixed as the index set of the given graph's edges.

The DAG characteristic function on the \ourmo\ search space is defined as follows:  

\begin{definition}\label{def:splr-hfunc}
    Let $\tilde{h}:\dom\mapsto\reals:A\to\trace(\exp(A))$ denote the
   exponential trace function. We define $h:\prodsp\mapsto\reals$ by  
    \begin{equation} \label{eq:def-h-ours}
        h(X,Y) = \trace(\exp (\sigma(\splr(X,Y)))), 
    \end{equation}
    where %
    $\sigma:\dom\to\dom$ is one of the following elementwise operators: 
\begin{align}
    & \sigma_2(Z):=Z\odot Z \quad \text{and}\quad %
\abs(Z):= \sum_{i,j=1}^d |Z_{ij}| e_i \trs[e_j].  \label{eq:def-po-abs}  
\end{align}
\end{definition}

Note that operators $\sigma_2$ and $\abs$~\eqref{eq:def-po-abs} are two natural choices that meet the conditions (i)--(ii) of \Cref{thm:dag-hfunc}, since they both produce a nonnegative surrogate matrix of the $d\times d$ matrix $\splr$ while preserving the support of $\splr$. %

\subsection{Representativity} %
In the construction of a \ourmo\ matrix~\eqref{eq:def-splr-rep}, the low-rank component of the model---$X\trs[Y]$ with $(X,Y)\in\prodsp$---has a rank smaller or equal to $r$ (equality attained when $X$ and $Y$ have full column ranks), and the subsequent sparsification operator $\po$ generally induces a change in the rank of the final matrix model $\po(X\trs[Y])$~\eqref{eq:def-splr-rep}. %
Indeed, the rank of $\splr=\po(X\trs[Y])$ depends on an interplay between $(X,Y)\in\prodsp$ and the discrete set 
$\Omega\in[d]\times [d]$. %
The following examples illustrate the two extreme cases of such interplay: 
\begin{itemize}
    \item[(i)] 
The first extreme case: let $\Omega$ be the index set of the edges of a sparse
graph $\grp_{\Omega}$, and let $(X,Y)\in\reals^{d\times 1}\times
\reals^{d\times 1}$ be the pair of matrices containing all ones, for $r=1$, then
the \ourmo\ matrix $\po(X\trs[Y])=\mathbb{A}(\grp_{\Omega})$, \ie, the
adjacency matrix of $\grp_{\Omega}$. Hence
$\rank(\po(X\trs[Y]))=\rank(\mathbb{A}_{\Omega})$, 
which depends solely on $\Omega$ and is generally much larger than $r=1$. %
\item[(ii)] The second extreme case: let $\Omega$ be the full $[d] \times [d]$
    index set, then $\po$ reduces to the identity map such that $\po(X\trs[Y])=X\trs[Y]$ and $\rank(\po(X\trs[Y]))=\rank(X\trs[Y])\leq r$ for any $(X,Y)\in\prodsp$. 
\end{itemize}

In the first extreme case above, optimizing \ourmo~\eqref{eq:def-splr-rep} for DAG learning boils down to choosing the most relevant edge set $\Omega$, which is an NP-hard combinatorial problem~\cite{chickering1996learning}.
In the second extreme case, the optimization of \ourmo~\eqref{eq:def-splr-rep} reduces to learning the most pertinent low-rank matrices $(X,Y)\in\prodsp$, which coincides with optimizing the \notears-low-rank~\cite{fang2020low} model. 

In this work, we are interested in settings between the two extreme cases above such that both $(X,Y)\in\prodsp$ and the candidate set $\Omega$ have sufficient degrees of freedom. Consequently, the representativity of \ourmo\ depends on both the rank parameter $r$ and $\Omega$. 

Next, we present a way of quantifying the representativity of \ourmo\ with respect to a subset $\setdag^\star$ of DAG matrices. %
The restriction to a subset $\setdag^\star$ is motivated by the revelation that certain types of DAG
matrices---such as those with many hubs---can be represented by low-rank
matrices~\cite{fang2020low}. 
\begin{definition}\label{assp}
Let $\setdag^\star\subset\setdag$ be a given set of nonzero DAG matrices. For 
$Z_0\in\setdag^\star$, let $\splr^*(Z_0)$ denote any \ourmo\ matrix~\eqref{eq:def-splr-rep} such that 
$
\|\splr^*(Z_0)-Z_0\|=\min_{(X,Y)\in\prodsp} \|\splr(X,Y)-Z_0\|$, %
then we define the relative error of \ourmo\ w.r.t. $\setdag^\star$ as
\[
\dparam=\max_{Z\in\setdag^\star}\big\{\frac{\infn[\splr^*(Z) - Z]}{\infn[Z]}\big\},\] %
where $\infn[Z]:=\max_{ij}|Z_{ij}|$ denotes the matrix max-norm. 
For $Z_0\in\setdag^\star$, $\splr^*(Z_0)$ is referred to as an $\dparam$-quasi DAG matrix. 
\end{definition}
Note that the existence of $\splr^*(Z_0)$ for any $Z_0\in\setdag^\star$ is guaranteed by the closeness of the image set of \ourmo~\eqref{eq:def-splr-rep}.

Based on the relative error above, the relevance of the DAG characteristic function is established from the following proposition (proof in \Cref{sec-app:repres}):
\begin{proposition}\label{prop:wellp}
    Given a set $\setdag^\star\subset\setdag$ of nonzero DAG matrices. For any
    $Z_0\in\setdag^\star$ such that $\infn[Z_0]\leq 1$, without loss of generality, 
    the minima of \[\min_{(X,Y)\in\prodsp} \|\splr(X,Y) - Z_0\|\] belong to the set 
    \begin{align}
    \label{eq:set-wellp}
    \{(X,Y)\in\prodsp: h(X,Y)  - d \leq \ca\dparam\}
    \end{align}
where 
$\dparam$ is given in \Cref{assp} 
and
$\ca = \big(C_1\|Z_0\|_0+\sum_{ij}[e^{\sigma(Z_0)}]_{ij}\big)\infn[Z_0]$ for a constant $C_1 \geq 0$. 
\end{proposition} 
\begin{remark} 
    The constant $\ca$ in \Cref{prop:wellp} can be seen as a
    measure of total capacity of passing from one node to other nodes, and therefore
    depends on $d$, $\infn[Z_0]$ (bounded by $1$), 
    the sparsity and the average degree of the graph of $Z_0$. For
    DAG matrices with sparsity $\rho_0\sim 10^{-3}$ and $d\lesssim 10^3$, one can expect that $\ca \leq Cd$ for a constant
    $C$.  $\hfill\square$ 
\end{remark}

The result of~\Cref{prop:wellp} establishes that, under the said conditions, a given DAG matrix $Z_0$ admits low rank approximations $\splr(X,Y)$ satisfying  $h(X,Y)-d \leq \epsh$ for a small enough parameter $\epsh$. In other words, the low-rank projection with a relaxed DAG constraint admits solutions.

The general case of projecting a non-acyclic graph matrix $Z_0$ onto a low-rank matrix under a relaxed DAG constraint is considered in next section.

\section{Scalable projection from graphs to DAGs}
\label{sec:projdag-algs}
Given a (generally non-acyclic) graph matrix $Z_0\in\dom$, let us consider the projection of $Z_0$ onto the feasible set~\eqref{eq:set-wellp}:
\begin{align}
    \label{prog:main-generic}
    & \min_{(X,Y)\in\reals^{d\times \rkval}\times\reals^{d\times \rkval}}
    \frac{1}{2}\fro{\splr(X,Y)-Z_0}^2  \quad\text{subject to~}  h(X,Y) - d \leq
    \epsh, %
\end{align}
where $\splr(X,Y)$ is the \ourmo\ matrix~\eqref{eq:def-splr-rep}, $h$ is given by~\Cref{def:splr-hfunc}, and $\epsh > 0$ is a tolerance parameter. 
Based on~\Cref{prop:wellp} and given the objective function
of~\eqref{prog:main-generic}, the solution to~\eqref{prog:main-generic}
provides a quasi DAG matrix closest to $Z_0$ and thus enables finding a projection of $Z_0$ onto the DAG matrix space $\setdag$. 
More precisely, we tackle problem~\eqref{prog:main-generic} using the penalty method and focus on the primal problem, for a given penalty parameter $\lambda>0$, followed by elementwise hard thresholding: 
\begin{align}
& (X^*,Y^*)  =\argmin_{(X,Y)\in\prodsp} h(X,Y) + \frac{1}{2\lambda}\fro{\splr(X,Y)- Z_0}^2, \label{prog:projdag}\\ 
& A^*  = \mathbb{T}_{\dparam}(\splr(X^*,Y^*)), \label{prog:hthres} 
\end{align}
where $\mathbb{T}_{\dparam}(z) = z\delta_{|z|\geq \dparam}$ is the elementwise hard thresholding operator. The choice of $\lambda$ and $\dparam$ is discussed in \Cref{ssec-app:choice-lambda} in the context of problem~\eqref{prog:main-generic}. 
\begin{remark} 
    \label{rmk:omega-projdag}
To obtain any DAG matrix closest to (the non-acyclic) $Z_0$, it is necessary to break the cycles in $\grp(Z_0)$ by suppressing certain edges. Hence, we assume that the edge set of the sought DAG is a strict subset of $\suppo(Z_0)$ 
and thus fix the candidate set $\Omega$ to be $\suppo(Z_0)$. $\hfill\square$ 
\end{remark}

Problems~\eqref{prog:main-generic} and~\eqref{prog:projdag} are nonconvex due to: (i) the matrix exponential trace $A\to\exptr[A]$ in $h$~\eqref{eq:def-h-ours} is nonconvex (as in \cite{NEURIPS2018_e347c514}), and (ii) the matrix product $(X,Y)\to X\trs[Y]$ in $\splr(X,Y)$~\eqref{eq:def-splr-rep} is nonconvex. %
In view of the DAG characteristic constraint of~\eqref{prog:main-generic}, 
the augmented Lagrangian algorithms of \notears~\cite{NEURIPS2018_e347c514} and
\notears-low-rank~\cite{fang2020low} can be applied for the same objective as problem~\eqref{prog:main-generic}; %
it suffices for \notears\ and \notears-low-rank to replace the LoRAM matrix $\splr(X,Y)$ in~\eqref{prog:main-generic} by the $d\times d$ matrix variable and the dense matrix product $X\trs[Y]$ respectively. 

However, the \notears-based methods of~\cite{NEURIPS2018_e347c514,fang2020low} have an $O(d^3)$ complexity due to the composition of the elementwise operations (the Hadamard product $\odot$) with the matrix exponential in the $h$ function~\eqref{eq:zheng18-thm1} and~\eqref{eq:def-h-ours}. 
We elaborate this argument in the next subsection and then propose a new computational method for computations involving the gradient of the DAG characteristic function $h$.

\subsection{Gradient of the DAG characteristic function and an efficient approximation}

\begin{lemma}\label{lemm:diffcalc-sigma}
For any $Z\in\dom$ and $\xi\in\tdom$, 
the differentials of
$\sigma_2$ and $\abs$~\eqref{eq:def-po-abs} are 
\begin{equation}\tag{\ref{eq:def-po-abs}b}
    \dop\sigma_2(Z)[\xi] = 2Z\odot\xi \quad\text{and}\quad 
    \hat{\dop}\abs(Z)[\xi] = \ssgn(Z)\odot \xi, 
\end{equation} 
where $\ssgn(\cdot)$ is the element-wise sign function such that $[\ssgn(Z)]_{ij} =
\frac{Z_{ij}}{|Z_{ij}|}$ if $Z_{ij}\neq 0$ and $0$ otherwise. %
\end{lemma} 

\begin{theorem} %
\label{lemm:diffcalc-lrnotears-hgen} 
The gradient of $h$~\eqref{eq:def-h-ours} is 
\begin{align}\label{eq:gradh-ours}
    \nabla h(X,Y) = (\smat Y, \trs[\smat]X)\in\prodsp, 
\end{align}
where $\smat\in\dom$ has the following expressions, depending on the choice of
$\sigma$ in~\eqref{eq:def-po-abs}: with $\splr:=\splr(X,Y)$ for brevity, 
\begin{align} 
    & \sodot = 2 (\trs[{\exp(\sigma_2(\splr))}]) \odot \splr,\label{eq:s-grad-h-odot} \\
    & \sabs= (\trs[{\exp(\abs(\splr))}]) \odot \ssgn(\splr). \label{eq:s-grad-h-abs} 
\end{align}

\end{theorem}

\begin{proof}
    From \Cref{prop:trexpm-divers}-(iii), the Fr\'echet derivative
    of the exponential trace $\tilde{h}$ at $A\in\dom$ is $\dop \tilde{h}(A) [\xi] =
        \trace(\exp(A) \xi)$ for any $\xi\in\reals^{d \times d}$. By the chain
        rule and \Cref{lemm:diffcalc-sigma},  
        the Fr\'echet derivative of $h$~\eqref{eq:def-h-ours} for $\sigma=\sigma_2$ is as follows: with $\splr=\po(X\trs[Y])$,  
        \begin{align}
        & \dop_X h(X,Y)[\xi]   =  \trace\big(\exp(\sigma(\splr))
        \dop\sigma(\splr)\dop\po(X\trs[Y])[\xi\trs[Y]]\big) \nonumber \\ 
            & = \trace\big(\exp(\sigma(\splr))
            \dop\sigma(\splr)[\po(\xi\trs[Y])]\big) \\
            & = 2 \trace\big(\exp(\sigma(\splr)) (\splr\odot \po(\xi\trs[Y]))\big)
            \nonumber\\
            & = 2 \trace\big(\exp(\sigma(\splr)) (\splr\odot (\xi\trs[Y]))\big) \label{eq:prf-po-is-proj} \\
            & = 2 \trace\big((\exp(\sigma(\splr)) \odot \trs[\splr]) (\xi\trs[Y])\big)
            \label{eq:tr-odot-flipp}\\
            & = 2 \trace\big(\trs[Y](\exp(\sigma(\splr)) \odot \trs[\splr]) \xi\big),
            \nonumber
        \end{align}
where~\eqref{eq:prf-po-is-proj} holds, \ie, $\po(X\trs[Y])\odot\po(\xi\trs[Y])=\po(X\trs[Y])\odot \xi\trs[Y]$
because $A\odot \po(B) = \po(A)\odot B$ (for any $A$ and $B$) and $\po^2=\po$, and~\eqref{eq:tr-odot-flipp} holds because $\trace(A(B\odot C)) =
        \trace((A\odot\trs[B])C)$ for any $A, B, C$ (with compatible sizes). 
        By identifying the above equation with
        $\trace(\trs[\nabla_X h(X,Y)]\xi)$, we have $\nabla_X
        h(X,Y) =
        \underbrace{2\big(\trs[{\exp(\sigma(Z))}] \odot
        \splr \big)}_{\smat} Y$, hence the
        expression~\eqref{eq:s-grad-h-odot} for $\smat$. The expression of
        $\sabs$ for $\sigma=\abs$ can be obtained using the same
        calculations and~(\ref{eq:def-po-abs}b). 
\end{proof}

The computational bottleneck for the exact computation of the gradient~\eqref{eq:gradh-ours} lies in the computation of the $d\times d$ matrix $\smat$~\eqref{eq:s-grad-h-abs} and is due to the difference between the Hadamard product and matrix multiplication; see \Cref{ssec-app:compgrad} for details. %
Nevertheless, we note that the multiplication $(\smat,X)\to \smat X$ is similar to the action of matrix exponentials of the form $(A,X)\to \exp(A)X$, which can be computed using only a number of repeated multiplications of a $d\times d$ matrix with the thin matrix $X\in\reals^{d\times r}$ based on Al-Mohy and Higham's results~\cite{al2011computing}.

The difficulty in adapting the method of~\cite{al2011computing} also lies in the presence of the Hadamard product in $\smat$~\eqref{eq:s-grad-h-odot}--\eqref{eq:s-grad-h-abs}. 
Once the sparse $d\times d$ matrix
$A:=\sigma(\splr)$ in~\eqref{eq:s-grad-h-odot}--\eqref{eq:s-grad-h-abs} is
obtained (using \Cref{alg:splr-prod}, \Cref{sec-app:algs}), the exact computation of $(A,C,B)\to (\exp(A)\odot C)B$, using the Taylor expansion of $\exp(\cdot)$ to a certain order $m_*$, is to compute $\frac{1}{k!}(A^k\odot C)B$ at each iteration, which inevitably requires the computation of the $d\times d$ matrix product $A^k$ (in the form of $A(A^{k-1})$) before computing the Hadamard product, which requires an $O(d^3)$ cost. %

To alleviate the obstacle above, we propose to use inexact\footnote{It is inexact because $((A\odot C)^{k+1})B \neq (A^{k+1}\odot C)B$.} %
incremental multiplications; see \Cref{alg:splr-expmv-inexa}. 

\begin{algorithm}[htbp]
    \caption{Approximation of $(A,C,B)\to (\exp(A)\odot C)B$\label{alg:splr-expmv-inexa}}  
\begin{algorithmic}[1]
    \REQUIRE{ $d\times d$ matrices $A$ and $C$, thin matrix $B\in\reals^{d\times \rkval}$, tolerance $\tol>0$ } \\
\ENSURE{$F\approx(\exp(A)\odot C)B\in\reals^{d\times \rkval}$}
\STATE Estimate the Taylor order parameter $m_*$ from $A$ \hfill\#
using~\cite[Algorithm~3.2]{al2011computing} 
\STATE Initialize: let $F= (I\odot C )B$ %
\FOR{$k = 1,\dots, m_*$} 
\STATE $B \leftarrow \frac{1}{k+1}(A\odot C)B$ %
\STATE $F \leftarrow F+B$ 
\STATE $B \leftarrow F$ 
\ENDFOR
\STATE Return $F$. 
\end{algorithmic} 
\end{algorithm}

In line~1 of \Cref{alg:splr-expmv-inexa}, the value of $m_*$ is obtained
from numerical analysis results of~\cite[Algorithm 3.2]{al2011computing};
often, the value of $m_*$, depending on the spectrum of $A$, is a bounded 
constant (independent of the matrix size $d$). %
Therefore, the dominant cost of \Cref{alg:splr-expmv-inexa} is $2m_*|\Omega|r\lesssim d^2r$, since each iteration (lines 4--6) costs $(2|\Omega|r+dr)\approx 2|\Omega|r$ flops. \Cref{tab:compl} summarizes this computational property in comparison with existing methods. %

\paragraph{Reliability of \Cref{alg:splr-expmv-inexa}.}
\label{ssec:reliability}
The accuracy of \Cref{alg:splr-expmv-inexa} with respect to the exact
computation of $(A,C,B) \to (\exp(A)\odot C)B$ depends notably on the scale of
$A$, since the differential $\dop\exp(A)$ at $A$ has a greater operator norm when the
norm of $A$ is greater; see \Cref{prop:expm-diff} in \Cref{sec-app:proofs}. 

To illustrate the remark above, we approximate $\nabla h(X,Y)$~\eqref{eq:gradh-ours} by \Cref{alg:splr-expmv-inexa} on random points of $\prodsp$ with
different scales, with $\Omega$ defined from the edges of a random sparse graph; the results are shown in \Cref{fig:tbm-accvsca}. 

\begin{figure}[htbp]
\centering 
\subfigure[Relative error]{\includegraphics[width=0.42\textwidth]{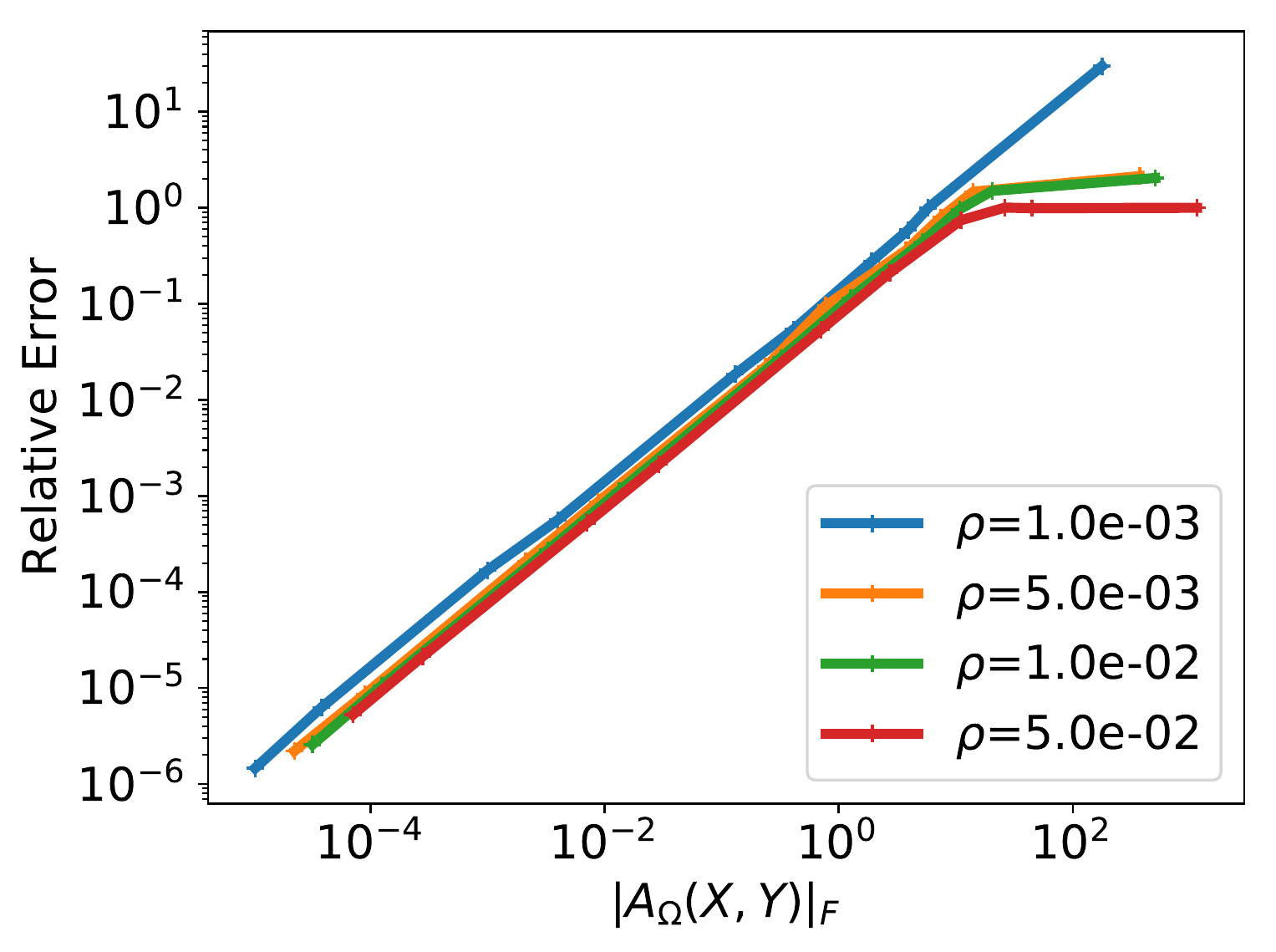}} 
\quad\qquad 
\subfigure[Cosine similarity]{\includegraphics[width=0.41\textwidth]{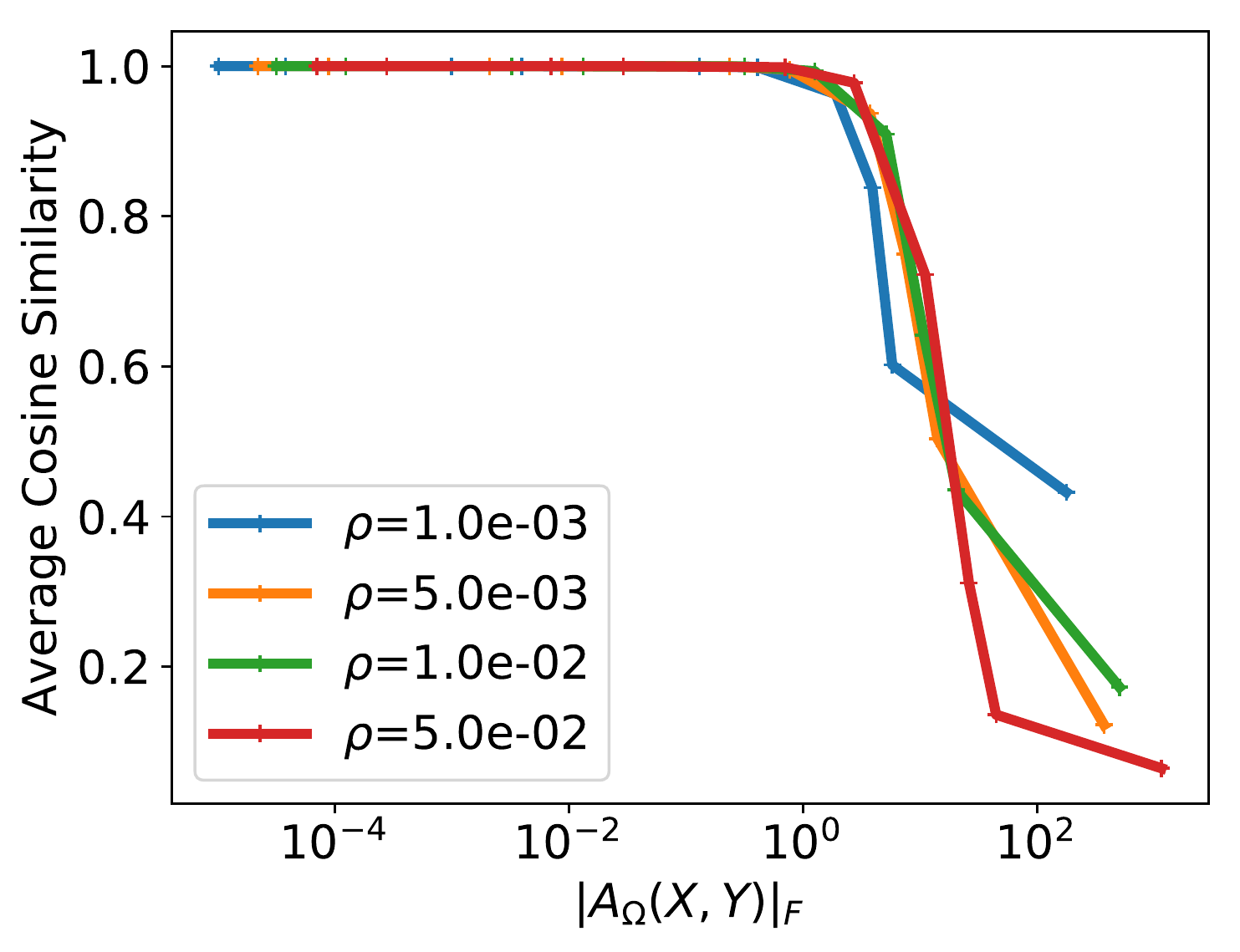}}
\caption{Average accuracy of \Cref{alg:splr-expmv-inexa} in approximating $\nabla h(X,Y)$~\eqref{eq:gradh-ours} at $(X,Y)$ with different scales and sparsity levels $\rho=\frac{|\Omega|}{d^2}$. Number of nodes $d=200$, $r=40$.} %
\label{fig:tbm-accvsca}
\end{figure}

We observe from \Cref{fig:tbm-accvsca} that \Cref{alg:splr-expmv-inexa} is reliable, \ie, having gradient approximations with cosine similarity close to $1$, when the norm of $\splr(X,Y)$ is sufficiently bounded. More precisely, for $c_0=10^{-1}$, \Cref{alg:splr-expmv-inexa} is reliable in the following set 
\begin{align}
    \label{eq:set-approx-safe}
    \mathcal{D}(c_0)=\{ (X,Y) \in \prodsp: \fro{\splr(X,Y)} \leq c_0 \}.
\end{align} 
 
The degrading accuracy of \Cref{alg:splr-expmv-inexa}
outside $\mathcal{D}(c_0)$~\eqref{eq:set-approx-safe} can nonetheless be avoided for the projection
problem~\eqref{prog:main-generic}, in particular, through {\it rescaling}
of the input matrix $Z_0$. Note that the edge set of any graph is invariant to the scaling of its weighted
adjacency matrix, 
and that any DAG 
$A^*$ solution to~\eqref{prog:main-generic} satisfies $\fro{A^*}\lesssim
\fro{Z_0}$ since $\suppo(A^*)\subset\suppo(Z_0)$ (see \Cref{rmk:omega-projdag}). Hence 
it suffices to rescale $Z_0$ with a small enough scalar, \eg, replace $Z_0$ with $Z_0'=
\frac{c_0}{10\fro{Z_0}}Z_0$, without loss of generality, in~\eqref{prog:main-generic}--\eqref{prog:projdag}. Indeed, this rescaling ensures that
both the input matrix $Z_0'$ and matrices like ${A^*}'=\frac{c_0}{10\fro{Z_0}}A^*$---a
DAG matrix equivalent to $A^*$---stay confined in the image (through \ourmo) of 
$\mathcal{D}(c_0)$~\eqref{eq:set-approx-safe},  
in which the gradient approximations by \Cref{alg:splr-expmv-inexa} are reliable.

\subsection{Accelerated gradient descent}\label{ssec:comp-gradh-inexact}

Given the gradient computation method (\Cref{alg:splr-expmv-inexa}) for the $h$
function, we adapt Nesterov's accelerated gradient descent~\cite{nesterov1983,Nesterov:2014:ILC:2670022} to solve~\eqref{prog:projdag}. 
The accelerated gradient descent is used for its superior performance than vanilla gradient descent in many convex and nonconvex problems while it also only requires first-order information of the objective function. %
Details of this algorithm for our LoRAM optimization is given in
\Cref{alg:solver-agd}. %

\begin{algorithm}[htpb]
\caption{Accelerated Gradient Descent of \ourmo~(\ouralg) \label{alg:solver-agd}} 
\begin{algorithmic}[1]
\REQUIRE{Initial point $x_0=(X_0, Y_0)\in\prodsp$, objective function $F=f+h$
with $h$ defined in~\eqref{eq:def-h-ours}, tolerance $\epsilon$.} 
\ENSURE{$x_t \in\prodsp$}
\STATE Make a gradient descent: $x_{1} = x_0 - s_0 \nabla F(x_0)$ for
an initial stepsize $s_0>0$  
\STATE Initialize: $y_0 = x_0$, $y_1 = x_1$, $t=1$. 
\WHILE{$\|\nabla F(x_t)\| > \epsilon$} 
\STATE Compute $\nabla F(y_t) =\nabla f(y_t) + \nabla h(y_t)$ \hfill\# using
\Cref{alg:splr-expmv-inexa} for $\nabla h(y_t)$~\eqref{eq:gradh-ours}
 \STATE Compute the Barzilai--Borwein stepsize: \label{alg:line-bb} 
        $s_t =  \frac{\|z_{t-1}\|^2}{\braket{z_{t-1},w_{t-1}}}$, 
where $z_{t-1} = y_t - y_{t-1}$ and $w_{t-1} = \nabla F(y_t) - \nabla F(y_{t-1})$. 
\STATE Updates with Nesterov's acceleration:  
        \begin{align}
            x_{t+1} & = y_{t} - s_t \nabla F(y_t), \label{eq:agd-upd1}  \\ 
            y_{t+1} & = x_{t+1} + \frac{t}{t+3} (x_{t+1} - x_t).  \nonumber 
        \end{align}
        \STATE $t=t+1$ 
\ENDWHILE
\end{algorithmic}
\end{algorithm}

Specifically, in line 5 of \Cref{alg:solver-agd}, the Barzilai--Borwein (BB)
stepsize~\cite{barzilai1988two} is used for the descent step~\eqref{eq:agd-upd1}. %
The computation of the BB stepsize $s_t$ requires evaluating the inner products $\braket{z_{t-1}, w_{t-1}}$ and
the norm $\|z_{t-1}\|$, where $z_{t-1},w_{t-1}\in\prodsp$; we choose the
Euclidean inner product as the metric on $\prodsp$: 
\[\braket{z, w} = \trace(\trs[z^{(1)}]w^{(1)}) +
\trace(\trs[z^{(2)}]w^{(2)}) \]
for any pair of points $z=(z^{(1)}, z^{(2)})$ and $w=(w^{(1)}, w^{(2)})$ on $\prodsp$. %
Note that one can always use backtracking line search based on the stepsize estimation (line~\ref{alg:line-bb}). %
We choose to use the BB stepsize directly since it does not require any evaluation of the objective function, and thus avoids 
the nonnegligeable costs for computing the matrix exponential trace in $h$~\eqref{eq:def-h-ours}. 
We refer to~\cite{carmon2018accelerated} for a comprehensive view on
accelerated methods for nonconvex optimization. 
Due to the nonconvexity of $h$~\eqref{eq:def-h-ours} (see \Cref{prop:trexpm-divers}) and thus~\eqref{prog:projdag}, 
we aim at finding stationary points of~\eqref{prog:projdag}. In particular, empirical
results in \Cref{ssec:exp-scal} show that the solutions by the proposed
method, with close-to-zero or even zero SHDs to the true DAGs, are close to global optima in practice. %

%
%
 
\section{Experimental validation}
\label{sec:exps}

This section investigates the performance (computational gains and accuracy loss) of the proposed gradient approximation method (\Cref{alg:splr-expmv-inexa}) and thereafter reports on the performance of the \ourmo\ projection~\eqref{prog:main-generic}, compared to \notears~\cite{NEURIPS2018_e347c514}. Sensitivity to the rank parameter $r$ of the proposed method is also investigated.


The implementation is available at \url{https://github.com/shuyu-d/loram-exp}.

\subsection{Gradient computations}
\label{ssec:exp-compgrad}
We compare the performance of \Cref{alg:splr-expmv-inexa} in gradient
approximations with the exact computation in the following settings: the number
of nodes $d\in \{100$, $500$, $10^3$, $2.10^3$, $3.10^3$, 
$5.10^3\}$, $r=40$, and the sparsity ($\frac{|\Omega|}{d^2}$) of the index set $\Omega$ tested are $\rho\in
\{10^{-3}, 5.10^{-3}, 10^{-2}, 5.10^{-2}\}$. 
%
%
The results shown in \Cref{fig:tbm-} are based on the computation of $\nabla h(X,Y)$~\eqref{eq:gradh-ours} on randomly generated points $(X,Y)\in\prodsp$, where $X$ and $Y$ are Gaussian matrices. %

\begin{figure}[htbp]
\centering 
\subfigure[Computation
time]{\includegraphics[width=0.55\textwidth]{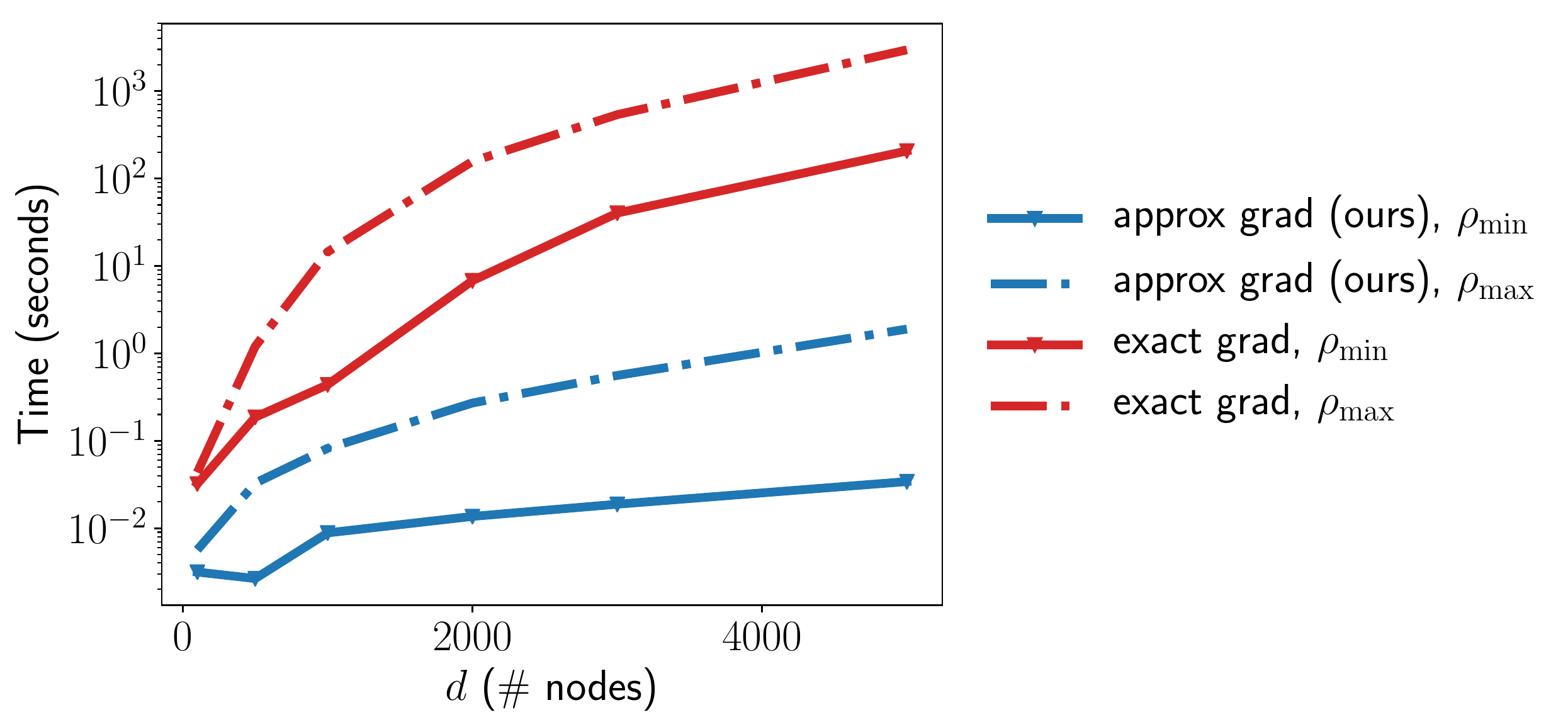}}
~
\subfigure[Accuracy]{\includegraphics[width=0.42\textwidth]{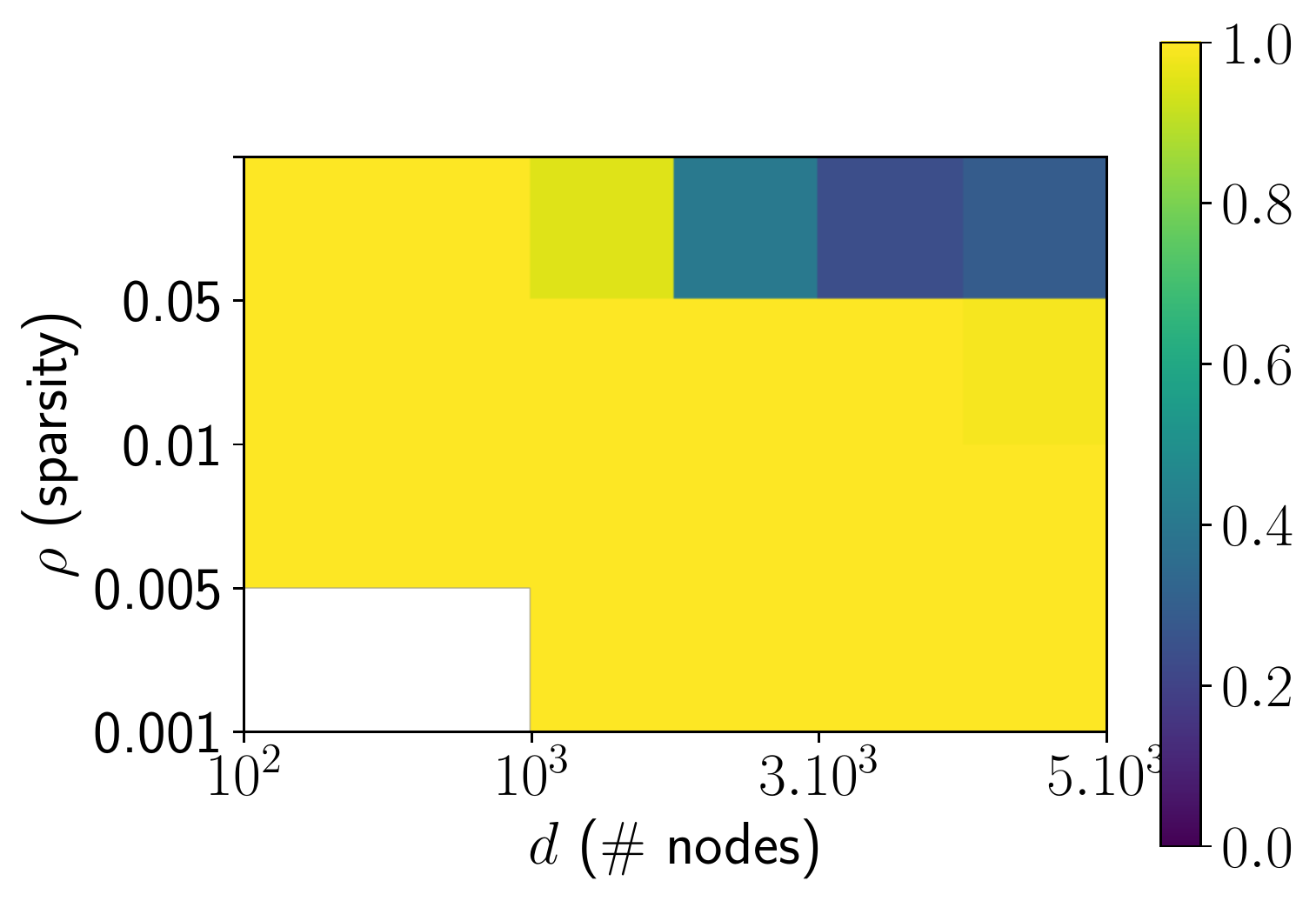}} 
\caption{(a): Runtime (in log-scale) for computing $\nabla h(X,Y)$. (b): Cosine similarities bewtween the approximate and the exact gradients.}
\label{fig:tbm-}
\end{figure}


From \Cref{fig:tbm-} (a), \Cref{alg:splr-expmv-inexa} shows a significant reduction in the runtime for computing the gradient of 
$h$~\eqref{eq:def-h-ours} at the expense of a very moderate loss of accuracy
(\Cref{fig:tbm-} (b)): the approximate gradients are mostly sufficiently
aligned with exact gradients in the considered range of graph sizes and sparsity levels. 

%
\subsection{Sensitivity to the rank parameter $r$}  
\label{ssec:exp-tsens}

In this experiment, 
we generate the input matrix $Z_0$ of problem~\eqref{prog:main-generic} as follows: 
\begin{align}
    & Z_0 = A^\star + E,  \label{eq:def-zstar-2} 
\end{align}
where $A^\star$ is a given $d\times d$ DAG matrix and $E$ is a graph containing
additive noisy edges that break the acyclicity of the observed graph $Z_0$. 

The ground truth DAG matrix $A^\star$ is generated from the acyclic
Erd\H{o}s-R\'enyi (ER) model (in the same way as
in~\cite{NEURIPS2018_e347c514}), with a sparsity rate $\rho\in \{10^{-3}, 5.10^{-3}, 10^{-2}\}$. 
The noise graph $E$ of~\eqref{eq:def-zstar-2} is defined as $E = \sigma_E \trs[A^\star]$, 
which consists of edges that create a confusion between causes and effects, since these edges are reversed, pointing from the ground-truth effects to their respective causes. 
We evaluate the performance of \ouralg\ in the proximal mapping computation~\eqref{prog:projdag} for $d=500$ and different values of the rank parameter $r$. 
In all these evaluations, the candidate set $\Omega$ is fixed to be $\suppo(Z_0)$; see \Cref{rmk:omega-projdag}. 

We measure the accuracy of the projection result by the false discovery rate (FDR, lower is better), false positive rate (FPR), true positive rate (TPR, higher is better), and the structural Hamming distance
(SHD, lower is better) of the solution compared to the DAG $\grp(A^\star)$. %

\begin{figure}[htpb]
\centering 
\includegraphics[width=0.41\textwidth]{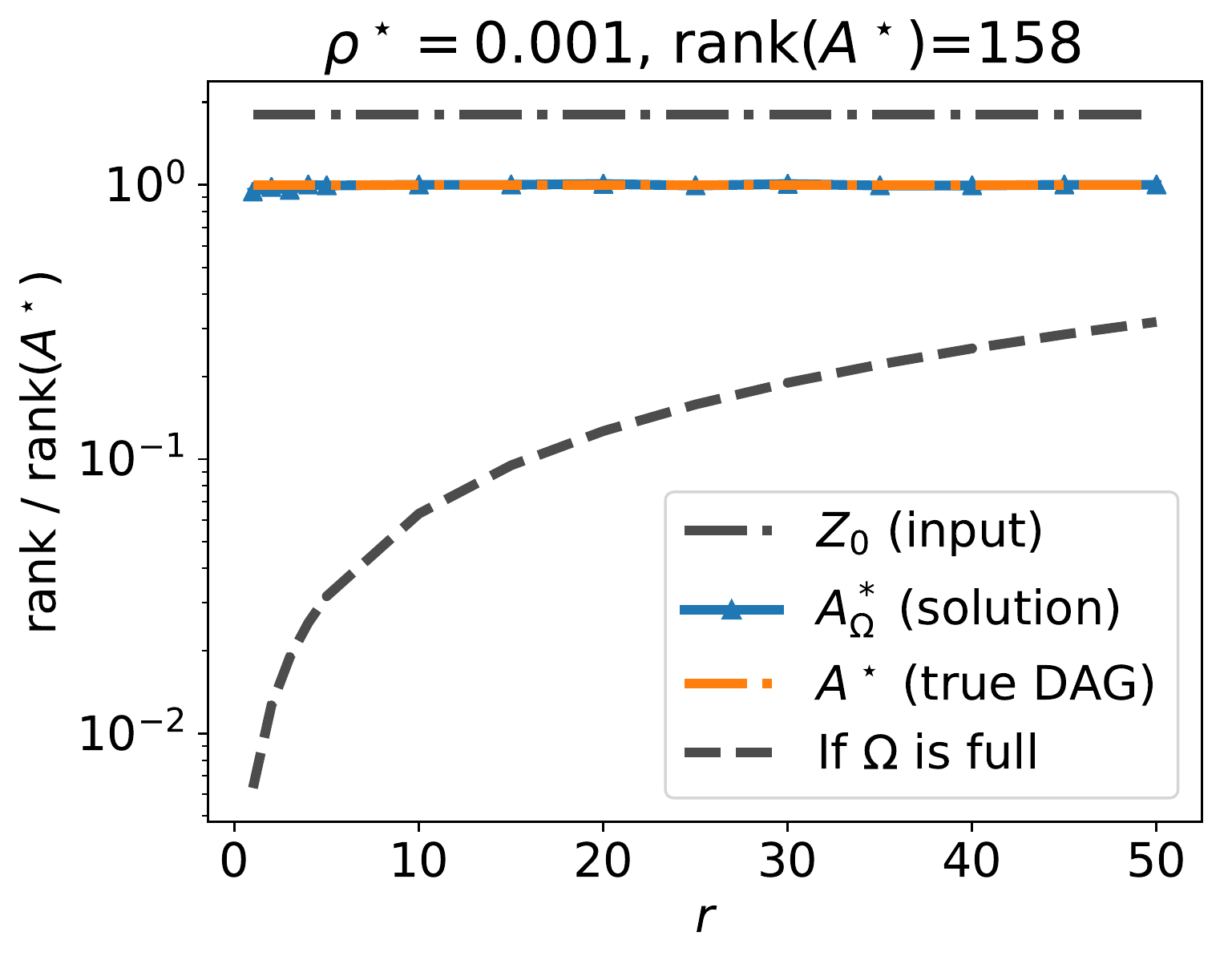}
\quad\qquad 
\includegraphics[width=0.41\textwidth]{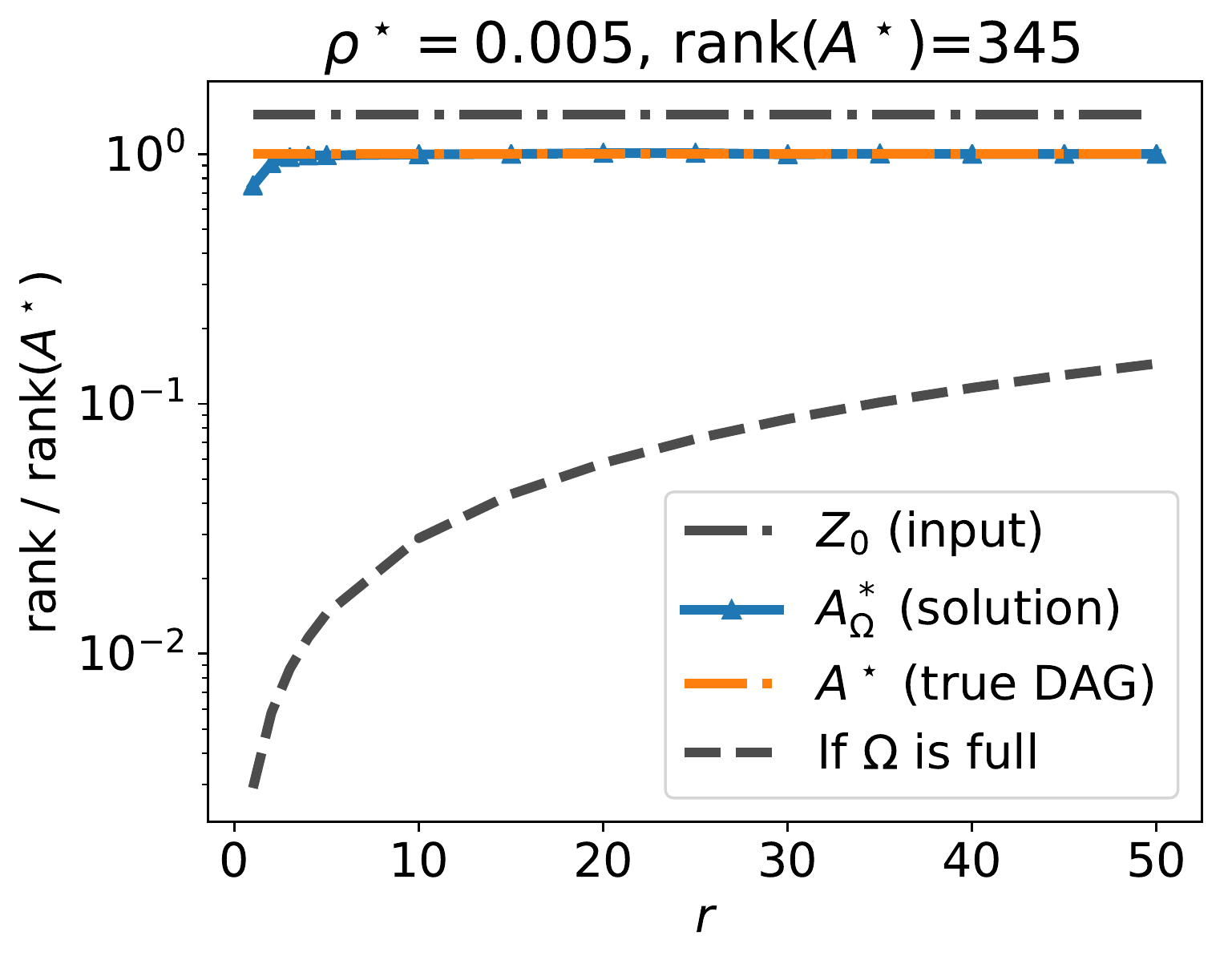} 
\\ 
\includegraphics[width=0.41\textwidth]{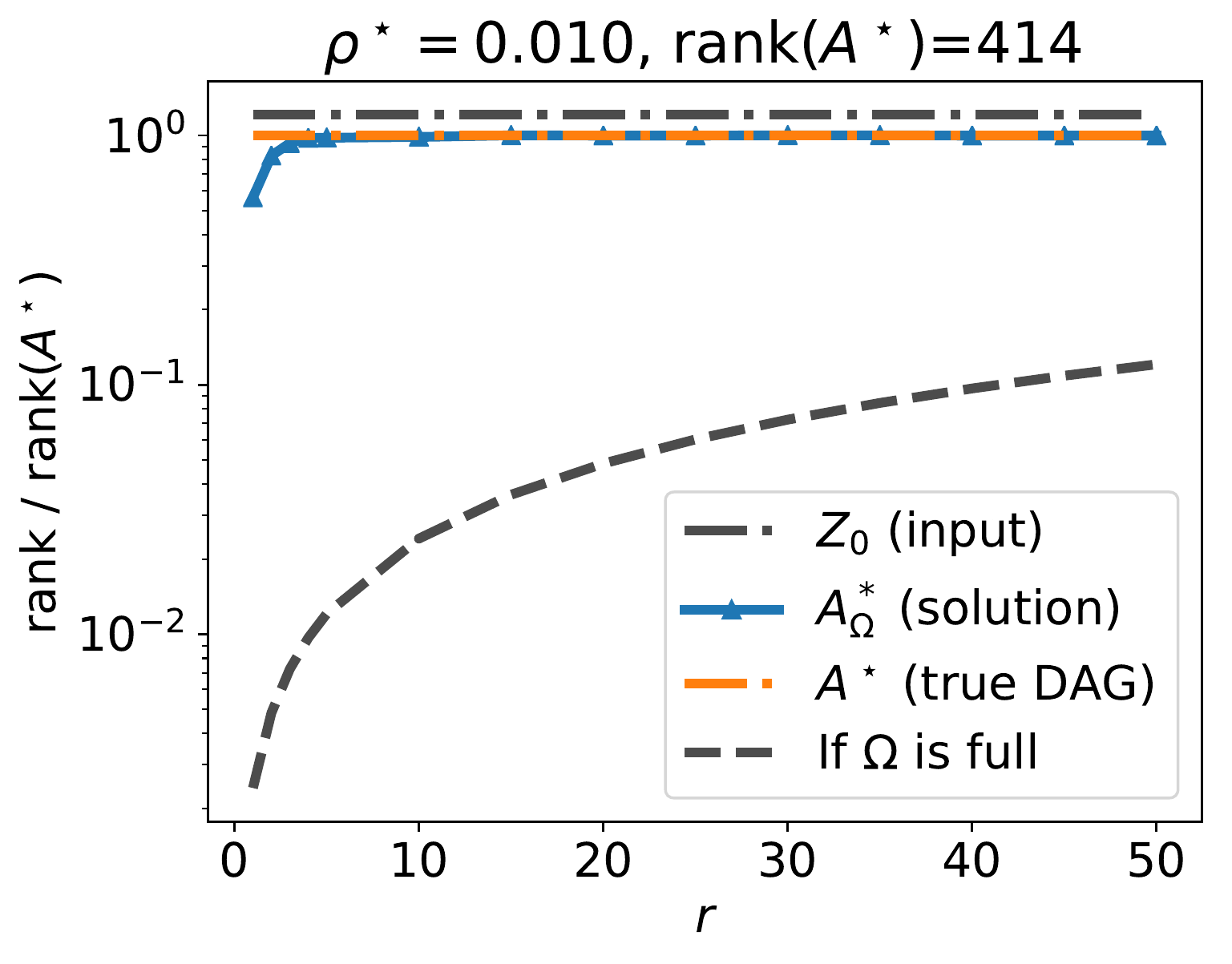}  
\quad\qquad 
\includegraphics[width=0.41\textwidth]{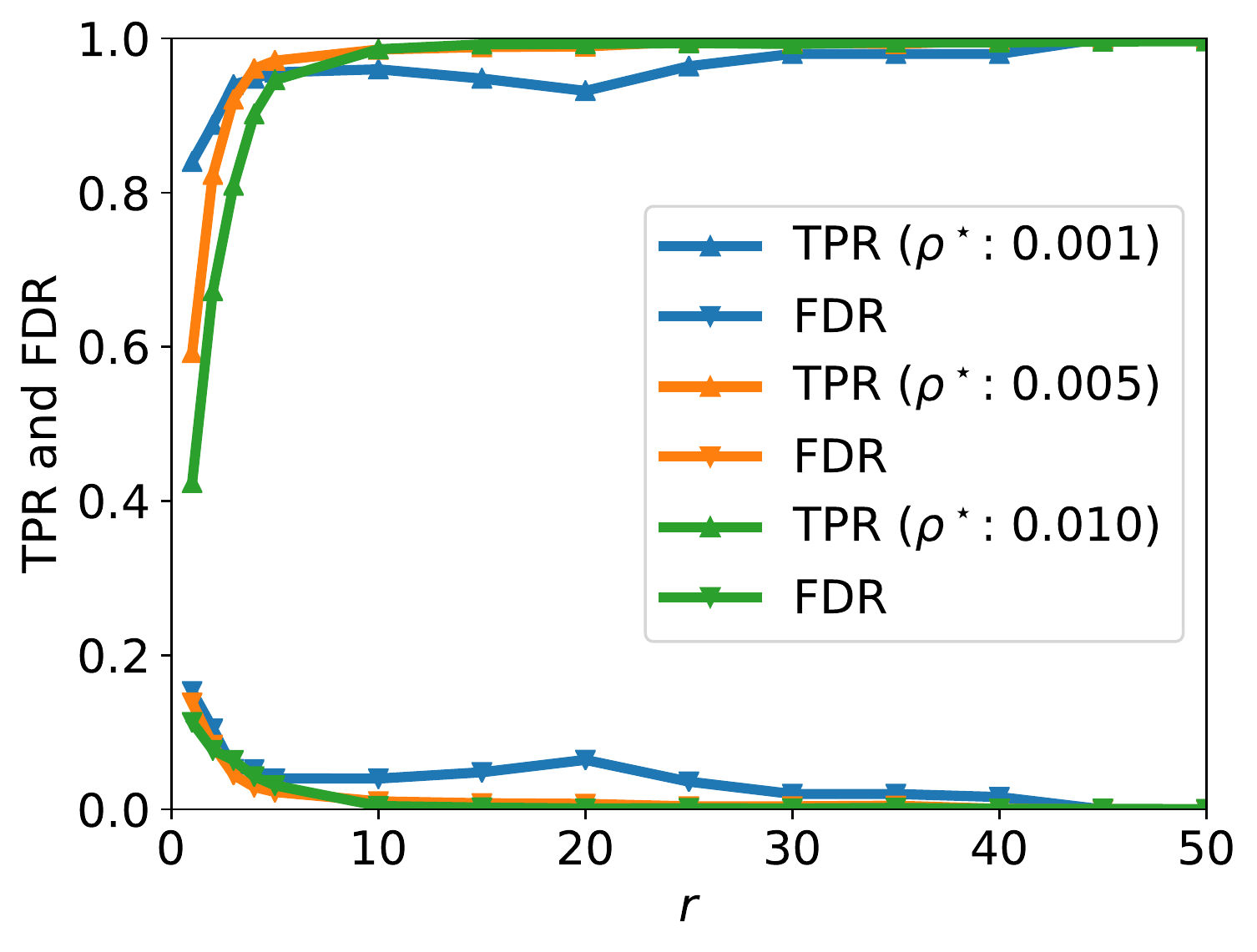}  
%
\caption{
Rank profiles and projection accuracy for different values of $r$ (number of columns of $X$ and $Y$ in~\eqref{eq:def-splr-rep}). The number of nodes $d=500$. The sparsity of the input non-DAG matrix $Z_0=A^\star +
\sigma_{E}\trs[{A^\star}]$~\eqref{eq:def-zstar-2} are $\rho^\star \in \{10^{-3}, 5.10^{-3}, 10^{-2}\}$. 
\label{fig:tsens-ranks}
}
\end{figure}

The results in \Cref{fig:tsens-ranks} suggest that:

\begin{itemize}
    \item[(i)] For each sparsity level, increasing the rank parameter $r$ generally improves the projection accuracy of the \ourmo. 
    \item[(ii)] While the rank parameter $r$ of~\eqref{eq:def-splr-rep} attains
        at most around $5$, which is only $\frac{1}{100}$-th the rank of the
        input matrix $Z_0$ and the ground truth $A^\star$, the rank of
        the solution $\splr^*=\splr(X^*,Y^*)$ attains the same value as
        $\rank(A^\star)$. This means that the rank representativity of
        the \ourmo\ goes beyond the value of~$r$. This phenomenon is
        understandable in the present case where the candidate set
        $\Omega=\suppo(Z_0)$ is fairly close 
        to the sparse edge set $\suppo(A^\star)$. 
    \item[(iii)] The projection accuracy in TPR and FDR (and also SHD, see \Cref{fig-app:tsens-ranks} in \Cref{ssec-app:exp-details}) of \ouralg\
        is close to optimum 
        on a wide interval $25 \leq r \leq 50$ of the tested ranks and are fairly stable on this interval.
\end{itemize}

\subsection{Scalability}
\label{ssec:exp-scal}

We examine two different types of noisy edges (in $E$) as follows. Case (a):
Bernoulli-Gaussian $E=E(\sigma_E, p)$, where
$E_{ij}(\sigma_E,p)\neq 0$ with probability $p$ and all nonzeros of $E(\sigma_E,p)$ are i.i.d.\ samples of
$\mathcal{N}(0,\sigma_E)$. Case (b): cause-effect confusions $E = \sigma_E \trs[A^\star]$ as in~\Cref{ssec:exp-tsens}. 

The initial factor matrices $(X,Y)\in\prodsp$ are random Gaussian matrices. %
For the \ourmo, we set $\Omega$ to be the support of $Z_0$; see
\Cref{rmk:omega-projdag}. %
The penalty parameter $\lambda$ of~\eqref{prog:projdag} is varied in
$\{2.0, 5.0\}$ with no fine tuning.

In case (a), we test with various noise levels for $d=500$ nodes. 
In case (b), we test on various graph dimensions, for $(d,\rkval)\in\{100,200,\dots,2000\}\times \{40, 80\}$. The results are given in \Cref{tab:benchm-projdag-0}--\Cref{tab:benchm-projdag-1} respectively. %

\begin{table}[htpb]
\footnotesize
\centering
\caption{Results in case (a): the noise graph is $E(\sigma_E,p)$ for $p = 5.10^{-4}$ and $d=500$. \label{tab:benchm-projdag-0}}
\begin{tabular}{c|cccc}
\hline\hline
\multirow{2}{*}{$\sigma_E$} & \multicolumn{4}{c}{\ourmo\ (ours) / \notears}  \\ 
&   {Runtime (sec)}  & TPR & FDR &  SHD \\
\hline                   
 0.1    & $\quad$1.34 / 5.78  $\quad$ & 1.0 / 1.0 $\quad$    & 9.9e-3 / 0.0e+0 $\quad$  & 25.0 / 0.0  \\ 
 0.2    & $\quad$2.65 / 11.58 $\quad$ & 1.0 / 1.0 $\quad$    & 9.5e-3 / 0.0e+0 $\quad$  & 24.0 / 0.0  \\ 
 0.3    & $\quad$1.35 / 28.93 $\quad$ & 1.0 / 1.0 $\quad$    & 9.5e-3 / 8.0e-4 $\quad$  & 24.0 / 2.0  \\ 
 0.4    & $\quad$1.35 / 18.03 $\quad$ & 1.0 / 1.0 $\quad$    & 9.9e-3 / 3.2e-3 $\quad$  & 25.0 / 9.0  \\ 
 0.5    & $\quad$1.35 / 12.52 $\quad$ & 1.0 / 1.0 $\quad$    & 9.9e-3 / 5.2e-3 $\quad$  & 25.0 / 13.0 \\ 
 0.6    & $\quad$2.57 / 16.07 $\quad$ & 1.0 / 1.0 $\quad$    & 9.5e-3 / 4.4e-3 $\quad$  & 24.0 / 11.0 \\ 
 0.7    & $\quad$1.35 / 18.72 $\quad$ & 1.0 / 1.0 $\quad$    & 9.9e-3 / 5.2e-3 $\quad$  & 25.0 / 13.0 \\ 
 0.8    & $\quad$1.35 / 32.03 $\quad$ & 1.0 / 1.0 $\quad$    & 9.9e-3 / 4.8e-3 $\quad$  & 25.0 / 15.0 \\ 
\hline\hline
\end{tabular}
\end{table}

\begin{table}[htpb]
\footnotesize
\centering
\caption{Results in case (b): the noise graph $E=\sigma
_E\trs[A^\star]$ contains cause-effect confusions for $\sigma_E = 0.4$. \label{tab:benchm-projdag-1}}
\begin{tabular}{c|c|cccc}
\hline\hline
\multirow{2}{*}{$(\lambda,\rkval)$} & \multirow{2}{*}{$d$}  & \multicolumn{4}{c}{\ourmo\ (ours) / \notears}  \\ 
& & {Runtime (sec)}  & TPR & FDR &  SHD \\
\hline                   
(5.0, 40)          & 100        & $\quad$ 1.82  / 0.67   $\quad$      & 1.00 / 1.00 $\quad$ & 0.00e+0  / 0.0 $\quad$ & 0.0  / 0.0    \\   
(5.0, 40)          & 200        & $\quad$ 2.20  / 3.64   $\quad$      & 0.98 / 0.95 $\quad$ & 2.50e-2  / 0.0 $\quad$ & 1.0  / 2.0    \\
(5.0, 40)          & 400        & $\quad$ 2.74  / 16.96  $\quad$      & 0.98 / 0.98 $\quad$ & 2.50e-2  / 0.0 $\quad$ & 4.0  / 4.0    \\
(5.0, 40)          & 600        & $\quad$ 3.40  / 42.65  $\quad$      & 0.98 / 0.96 $\quad$ & 1.67e-2  / 0.0 $\quad$ & 6.0  / 16.0   \\
(5.0, 40)          & 800        & $\quad$ 4.23  / 83.68  $\quad$      & 0.99 / 0.97 $\quad$ & 7.81e-3  / 0.0 $\quad$ & 5.0  / 22.0   \\
(2.0, 80)          & 1000       & $\quad$ 7.63  / 136.94 $\quad$      & 1.00 / 0.96 $\quad$ & 0.00e+0  / 0.0 $\quad$ & 0.0  / 36.0   \\
(2.0, 80)          & 1500       & $\quad$ 13.34 / 437.35 $\quad$      & 1.00 / 0.96 $\quad$ & 8.88e-4  / 0.0 $\quad$ & 2.0  / 94.0   \\
(2.0, 80)          & 2000       & $\quad$ 20.32 / 906.94 $\quad$      & 1.00 / 0.96 $\quad$ & 7.49e-4  / 0.0 $\quad$ & 3.0  / 148.0  \\
\hline\hline
\end{tabular}
\end{table}

\begin{figure}[htbp]
\centering 
\subfigure[Iterations of \ouralg\ (for $d=1000$)]{\includegraphics[width=0.46\textwidth]{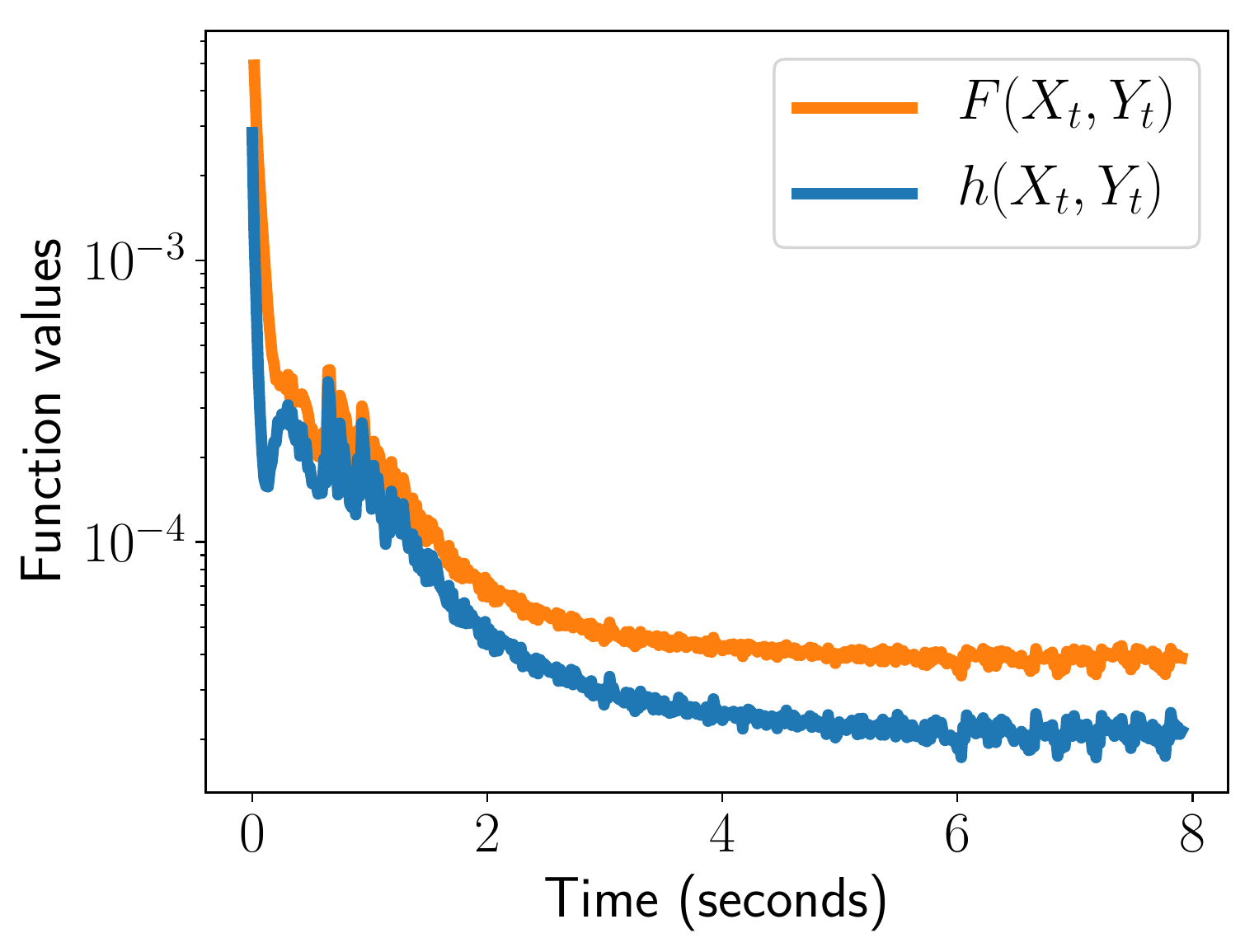}}
\qquad 
\subfigure[Runtime vs number of nodes]{\includegraphics[width=0.45\textwidth]{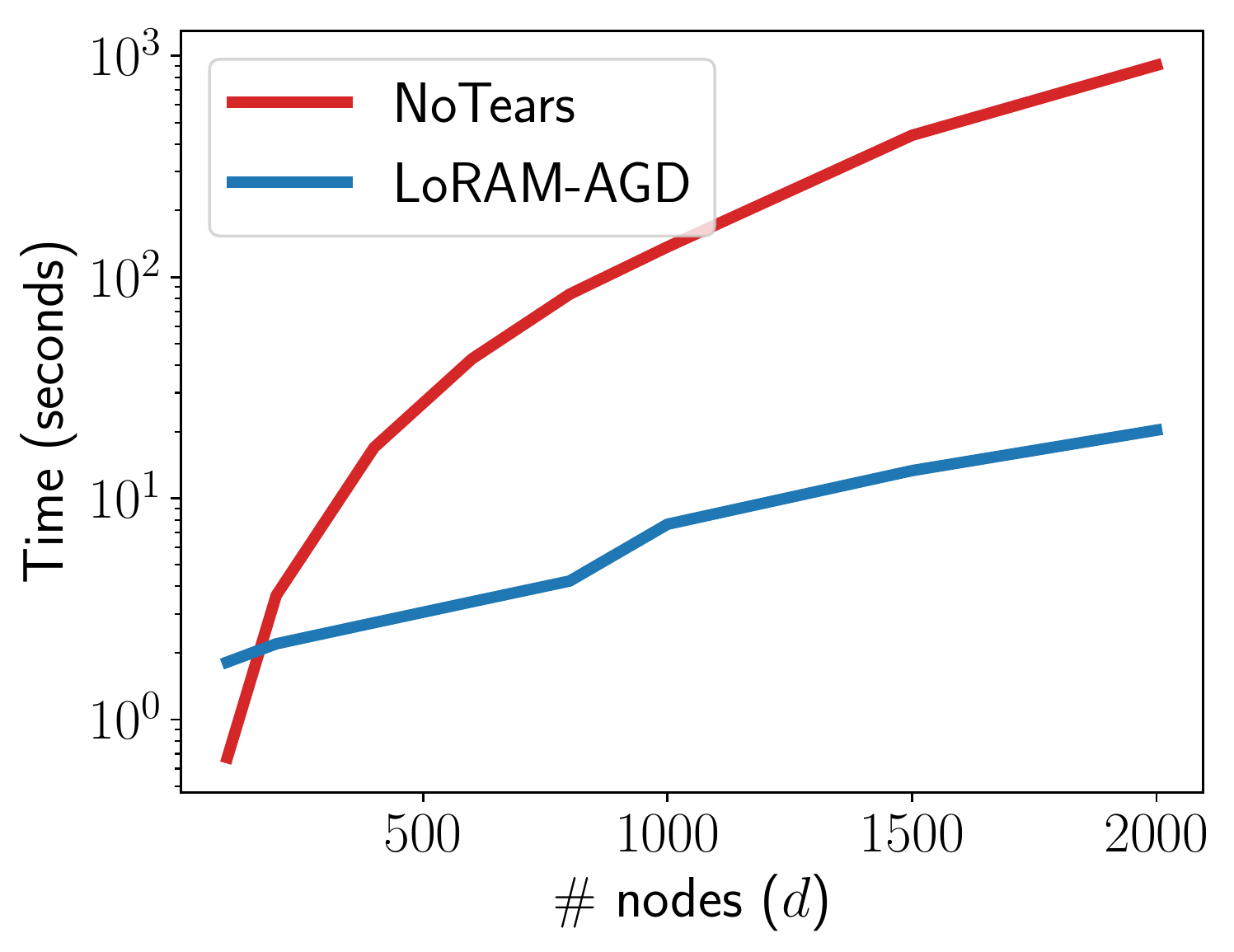}}  
\caption{(a): An iteration history of \ouralg\ for~\eqref{prog:projdag} with $d=1000$. (b): Runtime (in log-scale) comparisons for different number $d$ of nodes.}
\label{fig:tbm-iterh}
\end{figure}

The results in \Cref{tab:benchm-projdag-0}--\Cref{tab:benchm-projdag-1} show
that: 
\begin{itemize}
    \item[(i)] In case (a), the solutions of \ouralg\ are close to the ground truth
despite slighly higher errors than \notears\ in terms of
FDR and SHD. %
\item[(ii)] In case (b), the solutions of \ouralg\ are almost identical to the ground truth $A^\star$ In~\eqref{eq:def-zstar-2} in all performance indicators (also see \Cref{ssec-app:exp-details}). 
\item[(iii)] In terms of computation time (see \Cref{fig:tbm-iterh}), the proposed
\ouralg\ achieves significant speedups (around $50$ times faster when $d=2000$) compared to \notears\, and also has a smaller growth rate with respect
to the problem dimension $d$, showing good scalability. 
\end{itemize}

\section{Discussion and Perspectives}
This paper tackles the projection of matrices on DAG matrices, motivated by the identification of linear causal graphs. The line of research built upon the LiNGAM algorithms~\cite{shimizu2006linear,shimizu2011directlingam} has recently been revisited through the formal characterization of DAGness in terms of a continuously differentiable constraint by~\cite{NEURIPS2018_e347c514}. The \notears~approach of~\cite{NEURIPS2018_e347c514} however suffers from an %
$O(d^3)$ complexity in the number $d$ of variables, precluding its usage for large-scale problems.

Unfortunately, this difficulty is not related to the complexity of the model (number of parameters of the model): the low-rank approach investigated by \notears-low-rank~\cite{fang2020low} also suffers from the same $O(d^3)$ complexity, incurred in the gradient-based optimization phase. 

The present paper addresses this difficulty by combining a sparsification mechanism with the low-rank model and using a new approximated gradient computation. This approximated gradient takes inspiration from the approach of~\cite{al2011computing} for computing the action of exponential matrices based on truncated Taylor expansion. This approximation eventually yields a complexity of $O(d^2r)$, where the rank parameter is small ($r \leq C \ll d$).
The experimental validation of the approach shows that the approximated gradient entails no significant error with respect to the exact gradient, for \ourmo\ matrices with a bounded norm, in the considered range of graph sizes ($d$) and sparsity levels. The proposed algorithm combining the approximated gradient with a Nesterov's acceleration
method~\cite{nesterov1983,Nesterov:2014:ILC:2670022} yields gains of orders of magnitude in computation time compared to \notears\ on the same artificial benchmark problems. The approximation performance indicators reveal almost no performance loss for the projection problem in the setting of case (b) (where the matrix to be projected is perturbed with anti-causal links), while it incurs minor losses in terms of false discovery rate (FDR) in the setting of case (a) (with random additive spurious links). 

Further developments aim to extend the approach and address the identification
of causal DAG matrices from observational data. A longer term perspective is to extend \ourmo\ to the non-linear case, building upon the introduction of latent causal variables and taking inspiration from the non-linear independent component analysis and generalized contrastive losses~\cite{hyvarinen2019nonlinear}. Another perspective relies on the use of auto-encoders to yield a compressed representation of high-dimensional data, while constraining the structure of the encoder and decoder modules to enforce the acyclic property.

%
%

\section*{Acknowledgement}
The authors warmly thank Fujitsu Laboratories LTD who funded the first author, and in particular Hiroyuki Higuchi and Koji Maruhashi for many discussions.
%

\newcommand{\etalchar}[1]{$^{#1}$}

%

\appendix

\section{Proofs (Sections~\ref{sec:prelims}--\ref{sec:projdag-algs})}
\label{sec-app:proofs}

\paragraph{Proof of \Cref{thm:dag-hfunc}.} 
    By the condition (i), all entries of the matrix $B:=\sigma(A)$ are nonnegative. Hence, for all $k\geq 1$, 
    \begin{align}\label{eq:trb}
        \trace(B^k)\geq 0. 
    \end{align}
    By combining~\eqref{eq:trb} and (ii) that $\suppo(B)=\suppo(A)$,
    we confirm that $\trace(B^k)$ equals the sum of the weighted edge-sum along
    all $k$-cycles in the
    graph of $A$ (the graph whose adjacency matrix is $\suppo(A)$), according to classical graph
    theory (which can be obtained by induction). 
    Therefore, if $A$ is a DAG
    matrix, \ie, there are no cycles (of any length) in the
    graph of $A$, then $\trace(B^k)=0$ for all $k\geq 1$, hence
    $\trace(\exp(B))=\trace(I_{d\times d})=d$; and if $\trace(\exp(B))
    =d$, then $\sum_{k\geq 1}\frac{1}{k!}\trace(B^k) =0$, hence
    $\trace(B^k)=0$ (due to~\eqref{eq:trb}), which implies that there are no
    cycles of any length $k\geq 1$ in the graph of $A$. 
$\hfill\square$ 

\paragraph{Proof of \Cref{prop:trexpm-divers}.}
(i) All entries of any matrix $\bar{A} \in \dom_+$ are nonnegative,
hence $\trace(A^k) \geq 0$ for all $k\geq 1$. Therefore, $\exptr[\bar{A}] =
\trace(I_{d\times d}) + \sum_{k\geq 1} \frac{1}{k} \trace(\bar{A}^k) \geq d$. 
Moreover, \Cref{thm:dag-hfunc} shows that $\trace(\bar{A}) =d $ if and only if
$\bar{A}$ is a DAG matrix.

(ii) First, we show the nonconvexity of $\exptr$ on $\reals^{2\times 2}$ (for
$d=2$) by using the property (i) and \Cref{thm:dag-hfunc}: 
let 
$A=\begin{pmatrix} 0 & 1 \\ 0 & 0\end{pmatrix}$ and $B = \begin{pmatrix} 0 &
        0 \\ 1 & 0 \end{pmatrix}$. Notice that $A$ and $B$ are two different
        DAG matrices. 
Then, consider any matrix on the line segment between $A$ and $B$: 
        \[
        M_{\alpha}= \alpha A + (1-\alpha) B = 
        \begin{pmatrix} 0 & \alpha \\ 1-\alpha & 0\end{pmatrix} \text{~for~}
            \alpha \in (0,1). \]
            Then $M_{\alpha}$ is a weighted adjacency matrix of graph with
            $2$-cycles. Hence, $M_{\alpha}$ is nonnegative and is not a DAG matrix, it follows from the property (i) above that $\tilde{h}(M_{\alpha}) > d$. Property (i) also shows that $\tilde{h}(A) = \tilde{h}(B) = d$, since $A$ and $B$ are DAG matrices. Therefore, we have 
            \[\tilde{h}(M_{\alpha}) > d = \alpha\tilde{h}(A) +
            (1-\alpha) \tilde{h}(B).\]
            Hence $\tilde{h}$ is nonconvex along the segment $\{M_{\alpha}=\alpha A + (1-\alpha)B: \alpha \in (0,1)\} \subset \reals^{2\times 2}$. 
        To extend this example to $\reals^{d\times d}$, it suffices
        to consider a similar pair of DAG matrices $A'$ and $B'$ which differ
        only on their first $2\times 2$ submatrices, such that $A'_{:2,:2} = A$
        and $B'_{:2,:2} = B$. 

(iii)
First, the following definition and property are needed for the differential calculus of the matrix exponential function.  
\begin{definition}\label{def:expm-diff}
    For any pair of matrices $A, B\in\reals^{d\times d}$, the commutator
    between $A$ and $B$ is defined and denoted as follows: 
    \begin{align*}
        \ad_A(B) := [A, B]  = AB - BA. 
    \end{align*}
    The operator $\ad_A$ is a linear operator. More generally, the powers of
    the commutator $\ad_A(\cdot)$ are defined as follows: $(\ad_A)^0 = I$, and
    for any $k\geq 1$: 
    \begin{align*}
    (\ad_A)^k(B) =
    \underbrace{[A,\dots,[A,[A,B]]\dots]}_{k\text{~nested~commutators}}.
    \end{align*}
\end{definition}

\begin{proposition}[\cite{haber2018notes}~(Theorem 2.b)]\label{prop:expm-diff}
    For any $A\in\reals^{d\times d}$ and any $\xi\in\reals^{d\times d}$, the
    derivative of $t\mapsto \exp(A+t\xi)$ at $t=0$ is %
    $\ddt \exp(A+t\xi)|_{t=0} = e^A \tilde{f}(\ad_A)(\xi)$, 
    where $\tilde{f}(z) =
    \frac{1-e^{-z}}{z}=\sum_{n=0}^\infty \frac{(-1)^n}{(n+1)!}z^n$. 
\end{proposition}

    Now, we calculate the directional derivative of $\tilde{h}$ along a direction
    $\xi\in\reals^{d\times d}$: Note that the differential of $\trace(\cdot)$
    is itself anywhere, hence by the chain rule and \Cref{prop:expm-diff}, for
    any $\xi\in\reals^{d\times d}$, the directional derivative of $\tilde{h}$
    along $\xi$ is 
        \begin{align}
            & \ddt \trace(e^{A+t\xi})|_{t=0} =
            \trace\Big(\ddt(e^{A+t\xi})|_{t=0} \Big)\nonumber\\
            & = \trace(e^A \tilde{f}(\ad_A)(\xi)) \nonumber \\
            &=\trace(e^A \xi) + \frac{1}{(k+1)!} \sum_{k=1}^\infty \trace(e^A
            (-1)^k(\ad_A)^k(\xi)). \label{eq:tr-ads} 
        \end{align}
        Next, we show that last term of~\eqref{eq:tr-ads} is zero: note that
        for $k=1$, by rotating the three matrices in the trace function, we
        have 
        \begin{align}
        &\trace(e^A (-\ad_A(\xi))) = \trace(e^A (\xi A - A\xi)) =  \trace(A
        e^A \xi - e^A A\xi) \nonumber \\
            & ~ = \trace(\ad_A(e^A) \xi) = 0 \label{eq:ad-ea} 
        \end{align}
        for any $\xi$, where~\eqref{eq:ad-ea} holds because $\ad_A(e^A) = [A, e^A] = 0$ since
        $A$ and $e^A$ commute (for any $A$). 
        Furthermore, we deduce that for all $k\geq 2$: 
        \begin{align}
            & \trace\big(e^A (-1)^k (\ad_A)^k(\xi)\big)= \nonumber\\
             & (-1)^k\trace\big(e^A \ad_A((\ad_A)^{k-1}(\xi))\big) = 0, 
            \nonumber %
        \end{align} 
        because~\eqref{eq:ad-ea} holds for any direction including
        $\tilde{\xi}:=(\ad_A)^{k-1}(\xi)$. 
        Therefore, the equation~\eqref{eq:tr-ads} becomes
        $\ddt\tilde{h}(A+t\xi)|_{t=0} = \trace(e^{A}\xi)$, and through the identification
        $\ddt\tilde{h}(A+t\xi)|_{t=0} = \trace(\trs[\xi]\nabla \tilde{h}(A))$,
        the gradient reads $\nabla\tilde{h}(A) = \trs[(\exp(A))]$.  
$\hfill\square$ 

\paragraph{Proof of \Cref{prop:wellp}.}
\label{sec-app:repres}
    From~\Cref{assp}, the residual matrix
    $\xi:= \splr^* - Z_0\in\dom$ satisfies $\infn[\xi] \leq \dparam
    \infn[Z_0]$. Therefore, by the Taylor expansion and
    \Cref{prop:trexpm-divers}-(iii), there exists a constant $C_1 \geq 0$  such
    that 
    \begin{align}
        & \exptr[\sigma(\splr^*)] -d = \underbrace{\exptr[\sigma(Z_0)]-d}_{=0} %
        + \trace(e^{\sigma(Z_0)}\sigma(\xi))  + C_1 \fro{\sigma(\xi)}^2  \nonumber \\ 
        & = \trace(e^{\sigma(Z_0)}\sigma(\xi))  + C_1 \fro{\sigma(\xi)}^2 \label{eq:h-pert}\\ 
        & \leq \infn[\xi] (\sum_{ij} [e^{\sigma(Z_0)}]_{ij}) +
         C_1 \|Z_0\|_0 \infn[\xi]^2 \nonumber \\
        & \leq \dparam \infn[Z_0] (\sum_{ij} [e^{\sigma(Z_0)}]_{ij}) +
            C_1 \|Z_0\|_0  {\dparam}^2 \infn[Z_0]^2, \nonumber %
    \end{align}
    where~\eqref{eq:h-pert} holds since $Z_0$ is a DAG matrix (in view of
    \Cref{thm:zheng18-thm1}), and $\||Z_0\|_0$ is the number of nonzeros
    of $Z_0$. %
    It follows that, with relative error $\dparam\leq 1$ and
    $\infn[Z_0] \leq 1$ (without loss of generality): 
\[  
\exptr[\sigma(\splr^*)] -d \leq  \dparam \Big(C_1 \|Z_0\|_0 +
\sum_{ij}[e^{\sigma(Z_0)}]_{ij} \Big) \infn[Z_0],\]
which entails the result. %
$\hfill\square$ 

\paragraph{Proof of \Cref{lemm:diffcalc-sigma}.}
    %
    The Hadamard product $\odot$ is commutative, hence 
    \begin{align*}
        & \dop\sigma_2(Y)[\xi] = \ddt (Y+t\xi)\odot(Y+t\xi) |_{t=0} =
        2Y\odot\xi. 
    \end{align*}
    $\hat{\dop}\abs$~(\ref{eq:def-po-abs}b) is a subdifferential of~\eqref{eq:def-po-abs} since the
    sign function is a subdifferential of 
    the function $z\to |z|$ and $\abs(\cdot)$ is an elementwise matrix
    operator. 
$\hfill\square$ 

%
%

\section{Details of computational methods} 

\subsection{Computational cost of the exact gradient}
\label{ssec-app:compgrad}

The computation of the function value and the gradient of
$h$~\eqref{eq:def-h-ours} mainly includes two types of operations: (i)
the computation of the $d\times d$ sparse matrix $\sigma\circ\splr(X,Y)$ given
\Cref{def:splr-rep} and $\sigma$, and (ii) the computation
of the matrix exponential-related terms~\eqref{eq:s-grad-h-odot}
or~\eqref{eq:s-grad-h-abs}.

The exact computation of the gradient $\nabla h(X,Y)$ (\Cref{thm:dag-hfunc}) is
as follows:

(i) compute $\splr=\po(X\trs[Y])$~\eqref{eq:def-splr-rep}: this step has a cost
of $2|\Omega|\rkval$ flops; see \Cref{alg:splr-prod} in \Cref{sec-app:algs}.
The output $\splr$ is a $d\times d$ sparse matrix. A byproduct of
        this step is $\dop\sigma(\splr)$ ($2\splr$ or $\ssgn(\splr)$); \\ %
    (ii) compute $\exp(\sigma(\splr))$: %
    this step has a cost of $O(d^3)$; \\ 
    (iii) compute $M:=\exp(\sigma(\splr)) \odot \dop\sigma(\splr)$, which costs
    $|\Omega|$ flops. The output $M$ is a $d\times d$ sparse matrix;  \\ 
    (iv) Compute the sparse-dense matrix multiplication $(M,Y)\mapsto MY$: this step costs $2|\Omega|\rkval$ flops.
Therefore, the total cost is $O(d^3 + 4|\Omega|r)$. 

In scenarios with large and sparse graphs, \ie, $\rkval \ll d$ and $|\Omega| =
\rho d^2$ with $\rho \ll 1$, the complexity of computing the exact gradient of
the learning criterion thus is cubic in $d$, incurred by computing the matrix
exponential in step (ii); this step is necessary due to the Hadamard product
that lies between the action of $\exp(\splr)$ and the thin factor matrix $X$
(respectively, $Y$) in~\eqref{eq:s-grad-h-odot} or~\eqref{eq:s-grad-h-abs}: one needs to compute the $d\times d$ Hadamard product (before the matrix multiplication with $X$), which requires computing explicitly the $d\times d$ matrix exponential $\exp(\sigma(\splr))$.

For this reason, the \ourmo\ as well as \notears-low-rank of~\cite{fang2020low}
face a cubic complexity for the exact computation of the gradient of $h$~\eqref{eq:def-h-ours}, despite the significant reduction in the model complexity. 
\subsection{Computation of the LoRAM matrix}\label{sec-app:algs}

\begin{algorithm}[htpb]
    \caption{LoRAM matrix representation\label{alg:splr-prod}} 
\begin{algorithmic}[1]
    \REQUIRE{Thin factor matrices $(X,Y)\in\prodsp$, index set $\Omega=(I,J)$} 
    \ENSURE{$Z = \po(X\trs[Y])\in\dom$ and by products 
    $Z^{\mathrm{abs}}=\ssgn(Z)$,
    $Z^{\mathrm{sq}} =2Z$} 
    \STATE{Initialize: $Z = 0$}
    \FOR{$s = 1,\dots, |\Omega|$} 
    \FOR{$k = 1,\dots, \rkval$}  
     \STATE {$Z_{I(s),J(s)} = Z_{I(s),J(s)} + X_{I(s), k} Y_{J(s), k}$ } 
    \ENDFOR
    \ENDFOR
\STATE{Element-wise operations: $Z^{\mathrm{sq}} = 2 Z$, $Z^{\mathrm{abs}} =
\ssgn(Z)$} 
\end{algorithmic}
\end{algorithm}

\section{Experiments} 
\label{sec-app:exps}

\subsection{Choice of the parameters $\lambda$ and $\dparam$}
\label{ssec-app:choice-lambda}
The proposed algorithms (Algorithms~\ref{alg:splr-expmv-inexa}--\ref{alg:solver-agd}) are tested within the
penalty method~\eqref{prog:projdag} for a parameter $\lambda>0$. 

In the context of projection problem~\eqref{prog:main-generic}, the following
features are notable factors that influence the choice of an optimal $\lambda$:
(i) the type of graphs underlying the input matrix $Z_0$, and (ii) sparsity of the graph matrix $Z_0$. 
Note that the scale of $Z_0$ is irrelevant to the choice of $\lambda$ since we rescale 
the input graph matrix $Z_0\in\dom$ with $Z_0 \leftarrow
\frac{c_0}{10\fro{Z_0}}Z_0$ for a constant $c_0 = 10^{-1}$, without
loss of generality (see \Cref{ssec:reliability}), such that the scale of $Z_0$ 
always stay in the level of $\frac{c_0}{10}$ (in Frobenius norm). 
Therefore, we are able to fix a parameter set $\Lambda$ regarding the  
following experimental settings: (i) the type of graphs is fixed to be the 
ER acyclic graph set, same as in~\cite{NEURIPS2018_e347c514}, and (ii) the
sparsity level of $Z_0$ is fixed around $\rho \in \{10^{-3}, 5.10^{-3}, 10^{-2},
5.10^{-2}\}$. 
Concretely, we select the penalty parameter $\lambda$, among $N_\lambda=5$ values in the fixed set $\Lambda =
\{1.0, 2.0, \dots, 5.0\}$, with respect to the objective function value of the
proximal mapping~\eqref{prog:projdag}. 

On the other hand, the choice of the hard threshold $\dparam$
in~\eqref{prog:hthres} is straightforward, according to \Cref{assp} for the relative error of \ourmo\ w.r.t. a given subset $\setdag^\star$ of DAGs. For the set of graphs in our experiments, we use a moderate value $\dparam=5.10^{-2}$, which has the effect of eliminating only the very weakly weighted edges.

\subsection{Additional experimental details}
\label{ssec-app:exp-details} 

\begin{figure}[htpb]
    \centering 
    \includegraphics[width=0.45\textwidth]{rhos_tprfdr_vr}
    \quad 
    \includegraphics[width=0.46\textwidth]{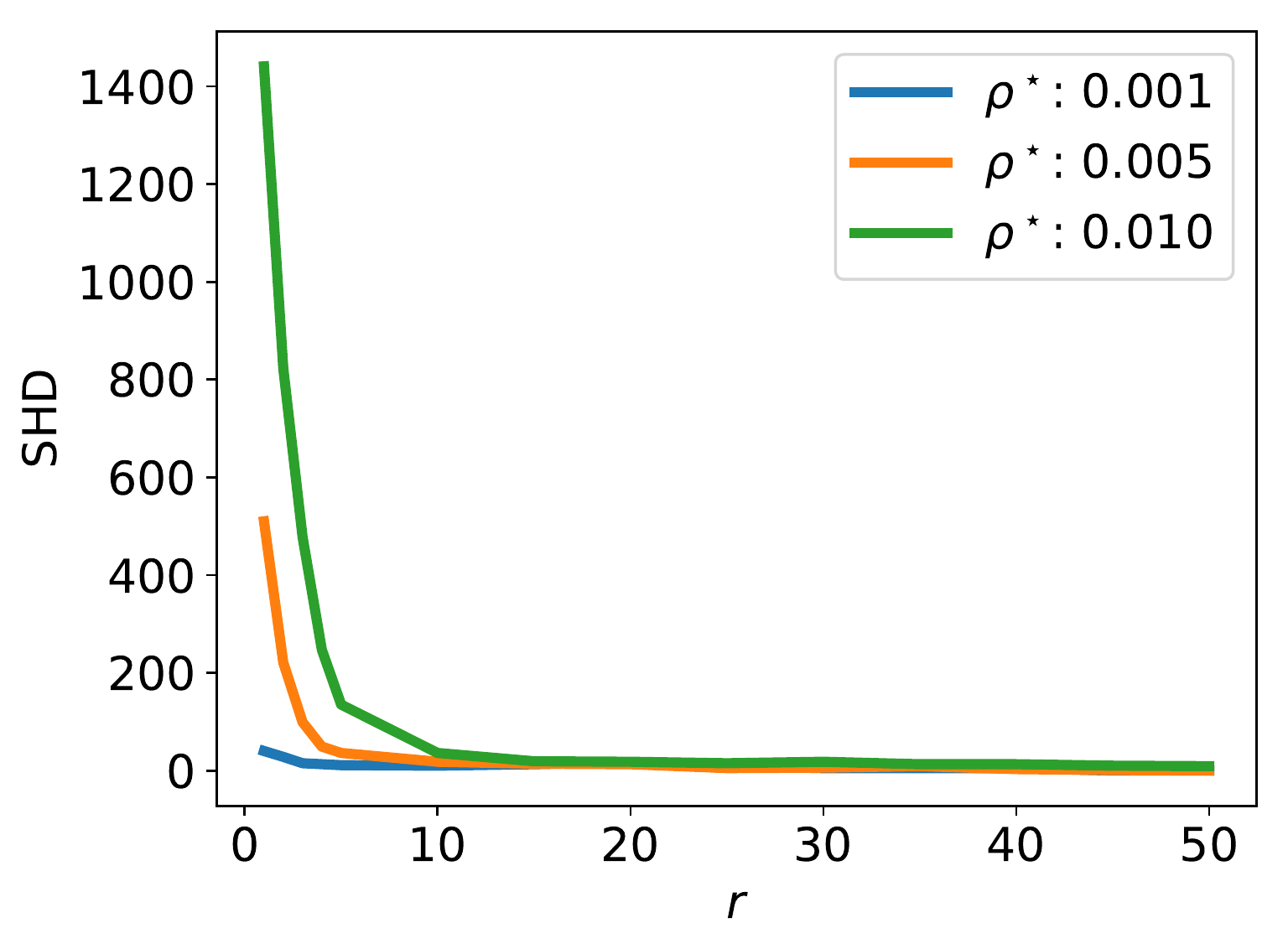}  
   \caption{
    Projection accuracies in TPR, FDR and SHD for different values of $r$ (number
    of columns of $X$ and $Y$ in~\eqref{eq:def-splr-rep}). The number of nodes
    $d=500$. The sparsity of the input non-DAG matrix $Z_0=A^\star +
    \sigma_{E}\trs[{A^\star}]$~\eqref{eq:def-zstar-2} are $\rho^\star \in
    \{10^{-3}, 5.10^{-3}, 10^{-2}\}$. 
    \label{fig-app:tsens-ranks}
    }
\end{figure}

\begin{table}[htpb]
\small 
\centering
\caption{Case (b): noise graph $E=\sigma
_E\trs[A^\star]$ creates cause-effect confusions, for $\sigma_E = 0.4$.
\label{tab-app:benchm-projdag-1}
(FDR, TPR, FPR, SHD) are the evaluation scores of the solution compared to the
ground truth DAG matrix $A^\star$.}
\begin{tabular}{c|c|c|ccccc}
\hline\hline
        Algorithm           & $(\lambda,\rkval)$  & $d$ & runtime       & FDR     & TPR   & FPR       & SHD \\ \hline              
\multirow{8}{*}{LoRAM}      & (5.0, 40)   & 100        & 1.82          & 0.0     & 1.0   & 0.0       & \bf{0.0}  \\  
                            & (5.0, 40)   & 200        & \bf{2.20 }    & 2.5e-2  & 0.98 & 5.0e-5          & \bf{1.0}  \\
                            & (5.0, 40)   & 400        & \bf{2.74 }    & 2.5e-2  & 0.98 & 5.0e-5          & \bf{4.0}  \\
                            & (5.0, 40)   & 600        & \bf{3.40 }    & 1.7e-2  & 0.98 & 3.4e-5          & \bf{6.0}  \\
                            & (5.0, 40)   & 800        & \bf{4.23 }    & 7.8e-3  & 0.99 & 1.6e-5          & \bf{5.0}  \\
                            & (2.0, 80)   & 1000       & \bf{7.63 }    & 0.0     & 1.0   & 0.0              & \bf{0.0}  \\
                            & (2.0, 80)   & 1500       & \bf{13.34}    & 8.9e-4  & 1.0   & 1.8e-6          & \bf{2.0}  \\
                            & (2.0, 80)   & 2000       & \bf{20.32}    & 7.5e-4  & 1.0   & 1.5e-6          & \bf{3.0}  \\
\hline
\multirow{8}{*}{NoTears}    & -           & 100       & \bf{0.67}      & 0.0      & 1.00      & 0.0     &   0.0  \\  
                            & -           & 200       &   3.64         & 0.0      & 0.95      & 0.0     &   2.0  \\
                            & -           & 400       &  16.96         & 0.0      & 0.98      & 0.0     &   4.0  \\
                            & -           & 600       &  42.65         & 0.0      & 0.96      & 0.0     &  16.0  \\
                            & -           & 800       &  83.68         & 0.0      & 0.97      & 0.0     &  22.0  \\
                            & -           & 1000      & 136.94         & 0.0      & 0.96      & 0.0     &  36.0  \\
                            & -           & 1500      & 437.35         & 0.0      & 0.96      & 0.0     &  94.0  \\
                            & -           & 2000      & 906.94         & 0.0      & 0.96      & 0.0     & 148.0  \\
                            
\hline\hline 
\end{tabular}
\end{table}

\begin{figure}[htpb]
\centering 
{\includegraphics[width=0.9\textwidth]{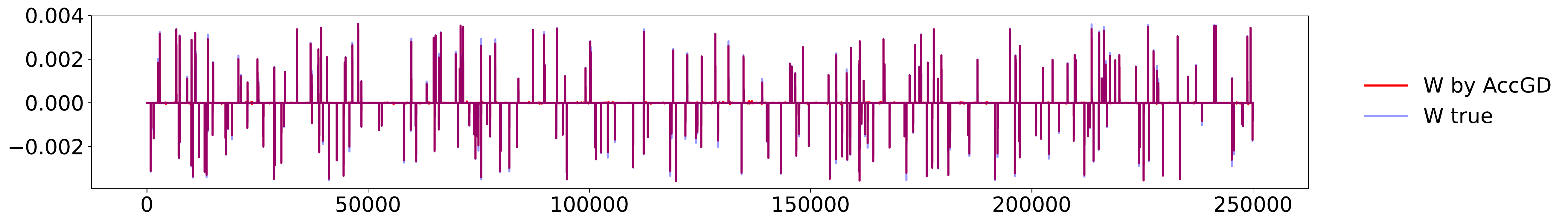}}
\caption{
First $2.5\times 10^5$ weighted edges of the solution (red) overlapping the edges of $A^\star$ (blue). %
Graph dimension $d=1000$, rank parameter $\rkval=80$. The recovery scores of the
solution are: TPR (higher is
better) $=1.0$, FPR (lower is better) $=0$, SHD (lower is better)) $=0$. 
The input graph matrix $Z_0=A^\star + \sigma_{E}\trs[{A^\star}]$.
\label{fig:iterhist-projdag-agd-1000}
}
\end{figure}

\end{document}